\documentclass[10pt]{article}
\usepackage{times}

\usepackage[linesnumbered,ruled,vlined]{algorithm2e}

\usepackage{mathtools}

\usepackage[utf8]{inputenc} % allow utf-8 input
\usepackage[T1]{fontenc}    % use 8-bit T1 fonts
\usepackage{hyperref}       % hyperlinks
\usepackage{url}            % simple URL typesetting
\usepackage{booktabs}       % professional-quality tables
\usepackage{amsfonts}       % blackboard math symbols
\usepackage{nicefrac}       % compact symbols for 1/2, etc.
\usepackage{microtype}      % microtypography
\usepackage{xcolor}

\usepackage{amsmath,amsfonts,amssymb,amsthm}
\usepackage{array}
\usepackage{longtable}
\usepackage{soul}
\usepackage{graphicx} 
\usepackage{float}
\usepackage{xcolor}
\usepackage{hyperref}
\hypersetup{
   colorlinks = false,
   pdfborder={0 0 0}}

\newcommand{\la}{\langle}
\newcommand{\ra}{\rangle}

\newcommand{\Norm}[1]{\left\|{#1}\right\|}

\def\commentson{0} %set 0 to disable all comments

\ifnum\commentson=1
\newcommand{\yv}[1]{\textcolor{red}{Yvonne: #1}}
\newcommand{\my}[1]{\textcolor{orange}{Mingyu: #1}}
\newcommand{\dnote}[1]{\textcolor{teal}{Dana: #1}}
\newcommand{\anti}[1]{\textcolor{magenta}{Antigoni: #1}}
\newcommand{\wm}[1]{\textcolor{blue}{ Min: #1}}
\else
\newcommand{\yv}[1]{}
\newcommand{\my}[1]{}
\newcommand{\dnote}[1]{}
\newcommand{\anti}[1]{}
\newcommand{\wm}[1]{}
\fi

\usepackage{tcolorbox}
\renewcommand\vec{\mathbf}

\newif\ifsubmission
\submissionfalse

\newif\ifsubmissionshort
\submissionshorttrue

\usepackage{subfigure}

\usepackage{adjustbox}

\usepackage{algorithmic}

\ifsubmissionshort{\theoremstyle{plain}
\newtheorem{theorem}{Theorem}[section]
\newtheorem{proposition}[theorem]{Proposition}
\newtheorem{lemma}[theorem]{Lemma}
\newtheorem{corollary}[theorem]{Corollary}
\newtheorem{claim}[theorem]{Claim}
\theoremstyle{definition}
\newtheorem{definition}[theorem]{Definition}
\newtheorem{assumption}[theorem]{Assumption}
\theoremstyle{remark}
\newtheorem{remark}[theorem]{Remark}}
\else{
\newtheorem{definition}{Definition}[section]
\newtheorem{theorem}{Theorem}
\newtheorem{lemma}[theorem]{Lemma}
\newtheorem{claim}[theorem]{Claim}

\newtheorem{corollary}{Corollary}

}
\fi

\usepackage[round]{natbib}

\usepackage{environ}
\newcommand{\acksection}{\section*{Acknowledgments and Disclosure of Funding}}
\NewEnviron{ack}{%
  \acksection
  \BODY
}

\RequirePackage{algorithm}
\RequirePackage{algorithmic}

\usepackage[preprint]{neurips_2023}

\begin{document}

%%%%%%%%%%%%%%%%%%%%%%%%%%%%%%%%%%%%%%%%%%%%%%%%%%%%%%%%%%%%%%%%%%%%%%%%%%%%%%

\title{Bounding the Excess Risk for Linear Models Trained on Marginal-Preserving, Differentially-Private, Synthetic Data}
\author{Yvonne Zhou \\ University of Maryland \\ College Park, MD 20742 \\ \texttt{skyzhou@umd.edu} \\
\And Mingyu Liang \\ University of Maryland \\ College Park, MD 20742 \\ \texttt{mliang@umd.edu} \\
\And Ivan Brugere \\ J.P. Morgan AI Research \\ New York, NY, 10017 \\
\texttt{ivan.brugere@jpmchase.com}\\
\And Danial Dervovic \\ J.P. Morgan AI Research \\ Edinburgh, UK \\
\texttt{danial.dervovic@jpmchase.com}\\
\And Antigoni Polychroniadou \\ J.P. Morgan AI Research \\AlgoCRYPT CoE\\ New York, NY, 10017 \\
\texttt{antigoni.polychroniadou@jpmorgan.com}\\
\And Min Wu \\ University of Maryland \\ College Park, MD 20742 \\ \texttt{minwu@umd.edu}\\
\And Dana Dachman-Soled \\ University of Maryland \\ College Park, MD 20742 \\ \texttt{danadach@umd.edu} }
\maketitle

\begin{abstract}
The growing use of machine learning (ML) has raised concerns that an ML model may reveal private information about an individual who has contributed to the training dataset. To prevent leakage of sensitive data, we consider using differentially-private (DP), synthetic training data instead of real training data to train an ML model. 
A key desirable property of synthetic data is its ability to preserve the low-order marginals of the original distribution. 
Our main contribution comprises novel upper and lower bounds on the excess empirical risk of linear models trained on such synthetic data, for continuous and Lipschitz loss functions.
We perform extensive experimentation alongside our theoretical results.

\end{abstract}

\section{Introduction}

\yv{comments: }

\yv{
Anti: 1. The laser sentence seemed unfinished since we do not add any conclusion of what we draw from the experiments - like verify our theory or something like this. 2. In the contributions section $\nu$ and $\delta$ are not defined it seems. 3. I don't see anywhere some mentions to our techniques for example why we use polynomial approximation etc. it would be good to mention something about our tech in the contributions section.}

\yv{Danial: For sure, I think the Intro is the place for more detailed discussion of points 1 and 2. Here is probably safer to say we are the first to provide theoretical support for using these models .. with respect to the downstream ML model performance.., as opposed to the marginal-preserving data itself. I agree with Min’s point from the meeting on Wednesday. Mechanism 1 stands on its own as a device for proving the bounds – we don’t have to make too much noise about it being a practical synthetic data generating algorithm.}

\sloppy
 Machine learning (ML) is extensively utilized at present, but a major concern is that the trained ML model may reveal private information about an individual who has contributed to the training dataset~\citep{fredrikson2014privacy, shokri2017membership, wang2022variational}.
 In response, various differentially-private (DP) machine learning methods, which typically add noise during the \emph{training} process, have been proposed in the literature~\citep{  Abadi_2016,6979031, papernot2017semisupervised, papernot2018scalable, DP-GD_0, yu2021gradient}. We refer to these methods as Training-Based Differentially-Private Machine Learning (Training-DPML).
 In contrast, in this work, we consider
 using differentially-private, synthetic training data instead of real training data to train the machine learning model. By doing so, 
one automatically achieves the guarantee that any models trained on the synthetic data are themselves differentially-private--i.e.~the weights associated with the trained models do not leak
information about any single individual in the dataset--without adding any additional noise during training. 
We therefore refer to the methods we study in this work as Preprocessing-Based Differentially-Private Machine Learning (Pre-DPML).

Pre-DPML techniques are an attractive option as opposed to Training-DPML techniques for several reasons.
First, Training-DPML algorithms require significant trust since the original sensitive data must be stored and handled throughout the training process and can only be discarded once all training has completed. 
Second, in Training-DPML techniques the privacy budget must grow with the total number of models trained and when the budget is depleted no further computations may be performed on the data. 

In contrast, when
Pre-DPML via 
%Pre-DPML and, in particular, 
DP synthetic
data generation is employed,
%allows us to circumvent the drawbacks listed above. 
the synthetic data is generated once and for all and the original sensitive data can be immediately discarded. Subsequently, one can perform any downstream task any number of times without requiring an increased privacy budget.
Further, one can safely use \emph{any} optimization algorithm out-of-the-box for training on the synthetic data (e.g.~second order methods or built-in Python optimization algorithms). 
Given the benefits of the Pre-DPML approach, our goal is to understand whether it is information-theoretically possible to generate synthetic data that achieves differential privacy and yields low excess risk in ML tasks. To answer this question, we first highlight a desirable property of DP synthetic data from the literature, known as \emph{marginal-preserving} synthetic data. The main results of this work, which we summarize in Section~\ref{sec:contributions}, provide novel upper and lower bounds on the excess empirical risk when training linear models on real versus marginal-preserving, synthetic data. To obtain a complete end-to-end analysis, we prove that DP and marginal-preserving synthetic data is attainable, whereas the marginal-preserving properties of prior DP mechanisms were heuristic.

\paragraph{Marginal-preserving synthetic data generation.} 
A $d$-th order marginal of a distribution
is the joint probability distribution of a subset of $d$ attributes.
Similarly, a $d$-th order marginal of a \emph{dataset} captures all possible statistics of the dataset for 
a subset of $d$ attributes.
Specifically, given a dataset, a marginal for a set of $d$ attributes is a vector that counts the number of occurrences of each combination of possible values of the attributes in the set.

The goal of synthetic data generation algorithms is to produce a synthetic dataset that closely matches the statistics of the original dataset.
In marginal-preserving synthetic data generation, the synthetic data preserves the statistics of a target set of marginals, as closely as possible. 

\paragraph{Marginal-preserving approach for DP synthetic data.}
Various marginal-preserving and \emph{differentially private}
synthetic data generation algorithms have been proposed in the literature, such as 
PrivBayes\citep{PrivBayes},  PrivSyn\citep{PrivSyn}, PrivMRF\citep{PrivMRF}, PEP and GEM \citep{PEP}, Private-PGM \citep{Graphical-model}, and AIM\citep{aim}.
Typically, the quality of the synthetic data has been measured in terms of the ability to accurately respond to statistical queries, even if the queries involve sets of attributes that were not contained in the target set of marginals. For example, in prior work, the synthetic data was evaluated by comparing its marginals
with the marginals of the true data for random triples of attributes, or by examining how well the
synthetic data preserved random high-order conjunctions~\citep{DBLP:journals/corr/abs-2108-04978}.

However, to our knowledge, research on the utility error of downstream tasks trained on synthetic data remains limited. While \cite{li2023statistical}'s work made some initial strides in analyzing the utility of downstream tasks, it relied on certain strong assumptions. For instance, they assumed that the data distribution can be represented as a Bayesian network with a degree no greater than $k$, in which case the variation distance stemming from high-order terms can be omitted. Alternatively, they were able to remove this assumption, but in this case the error grows exponentially to the dimension. In contrast, our bound applies to any data distribution by using a polynomial approximation of the loss function in the analysis. This essentially allows us to bound the excess risk stemming from high-order marginals, \emph{without imposing assumptions on the data distribution}. 
Additionally, they focused on training ML models with norm-bounded loss functions and utilized a specific marginal-based mechanism (PrivBayes). In contrast, our goal is to assess the quality of ML models trained on any continuous and Lipschitz loss function, and employs any marginal-based mechanisms.

\subsection{Our Contributions} \label{sec:contributions}
Our paper focuses on investigating the excess empirical risk (measured w.r.t.~the real dataset) of training linear models on marginal-preserving synthetic data
that approximately preserves the $d$-th order marginals of the real dataset. We present both theoretical and experimental results.

\begin{itemize}
     
\item In Section \ref{ssec:upperbound_loss_only}, we upper bound the excess empirical risk, as long as the low-order marginals of the synthetic data are sufficiently close to the real marginals.
We consider the setting where the 
dataset is scaled so that all $m$-dimensional datapoints lie in the $m$-dimensional unit ball and where we optimize the weights $\vec{w}$ over the unit ball.
In Theorem \ref{thm:ASD_gen}
we demonstrate that if the $\ell_1$ distance of all marginals up to order $d$ of the real and synthetic data is at most $\nu$, then for any continuous and $O(1)$-Lipschitz loss function, the difference in cost is upper-bounded by  $O(1/\sqrt{d-1} + (3m)^{d-1}\nu/n)$, where $n$ is the number of samples in both datasets.
Additionally, in Theorem \ref{thm:ASD_log}, we show that for logistic regression specifically, we achieve a tighter upper bound of $O(1/(d-1) + (3m)^{d-1}\nu/n)$.
\item 
 In Section \ref{ssec:upperbound_loss_priv}, 
we give an outline of an information-theoretic mechanism that generates $(\epsilon, \delta)$-differentially private synthetic data with a bounded $\ell_1$ difference of 
 $\frac{4 m^{d/2}l^d\sqrt{2\ln(1.25/\delta)(\ln(2)(1+\lambda)+d\ln(ml))}}{\epsilon}$ 
 except for $2^{-\lambda}$ probability, where $l$ is the maximum domain size of any attribute.
Substituting this bound into $\nu$ in the aforementioned Theorems, implies that as the size of the database $n$ goes to infinity, the excess empirical risk is dominated by $O(\frac{1}{\sqrt{d-1}})$ for general continuous and $O(1)$-Lipschitz loss functions, and dominated by $O(\frac{1}{d-1})$ for logistic regression.
In practice, various efficient DP algorithms can heuristically preserve the marginals. However, there is a lack of conclusive proof regarding the attainability of a specific $\ell_1$ bound for all input datasets. We conduct experiments and report the average $\ell_1$ distance, over selected queries,  achieved in practice for multiple datasets in Section~\ref{sec:experiment Synthetic data in varied privacy budgets}.
\item 
In Section \ref{sec:lower_bound}, we lower bound the excess empirical risk and demonstrate that for a specific range of parameter choices, we obtain a nearly tight match to the upper bound: $\Omega (\frac{1}{\ln^3(n)})$ versus $O (\sqrt{\frac{\ln(\ln(n))}{\ln(n)}})$. 
%It is important to emphasize that 
Our lower bound asserts the existence of a particular data distribution for which no marginal-preserving synthetic data algorithm, even if inefficient, can significantly outperform the upper bound. This, however, does not eliminate the possibility of better performance for real-life data distributions. Indeed,  in Section \ref{sec:exp}, our experimental results surpass the outcomes predicted by our lower bound.
 %Therefore, 
 Exploring reasonable assumptions on data distributions that allow bypassing the lower bound and obtaining improved upper bounds is an interesting future direction.
\item 
 We performed extensive experimentation, and the results can be found in Section~\ref{sec:exp}. To summarize our findings, we observed that, when with $(2,\frac{1}{n^2})$-DP, 
% the differential privacy parameters are $(\epsilon=2, \delta = \frac{1}{n^2} )$, 
the accuracy of the model trained on the marginal-preserving, DP synthetic data drops by less than 1\% compared to the real data, and the excess empirical risk is less than 0.02. The exception is the Heart dataset, which exhibits a 2.2\% drop in accuracy and  0.032 excess empirical risk, likely due to its considerably smaller dataset size.

\end{itemize}

\subsection{Related Work}
Private  stochastic gradient descent (SGD) was first introduced by Song et al. \citep{private_SGD_1}, and was subsequently enhanced in~\citep{private_SGD_2} and
\citep{Abadi_2016}.
DP-SGD modifies stochastic gradient descent by clipping per-sample gradients for sensitivity control and by injecting noise to aggregated batch gradients at each intermediate update.
Researchers have explored the application of DP-SGD and its variants~\citep{DP-GD_0} to various tasks~\citep{DP-SGD_FL_language, DP-SGD_NLU,DP-SGD_image,DP-SGD_medical}, and frameworks such as Distributed/Federated Learning ~\citep{DP-SGD_FL_language, DP-SGD_FL_medical,DP-SGD_fairness_decentralized}.
In contrast to DP-GD/DP-SGD, recent studies \citep{2nd_order_1, 2nd_order_2} suggest introducing noise to the \textit{Hessian} of the loss function rather than the gradient.  This technique allows the realization of differentially private optimization via second-order methods, which demonstrate a faster convergence rate than first-order methods such as gradient descent.
Another noteworthy DP-ML method is Private Aggregation of Teacher Ensembles(or PATE)~\citep{papernot2017semisupervised, papernot2018scalable}. PATE proposes training an ensemble of non-private models (teachers), obtaining their predictions on a small set of unlabeled public data, and central aggregating predictions with noise. The labeled public data points are then used to train a student model. 
It is apparent that deploying PATE would consume computational overhead for training multiple teacher models in order to train a single student model. Moreover, it crucially presupposes the availability of public, unlabeled data.

\section{Notation and Background}\label{sec:bg}
We use $[n]$ to denote $\{1,2,\dots,n\}$ and boldface variable to represent a vector, e.g., $\vec{v}$ and $\vec{h}$. Moreover, we use $\vec{v}[i]$ to denote the $i^{th}$ entry of the vector, and $\vec{v}[q] = (\vec{v}[j])_{j\in q}$ to denote the subvectors containing entries in set $q$.

\subsection{Data and Marginals}

\paragraph{Data.} A dataset $D$ is a multiset of $n$ samples, each can be represented as $\vec{v}=(\vec{x},y)\in V$, where $\vec{x} = (x_1,\dots,x_m)$ is a vector of $m$ features and $y$ is the corresponding label/class for the sample. For convenience, we may also refer to $y$ as the $(m+1)^{th}$ feature. For $j\in[m+1]$, let $\Omega_j$ denote the domain of possible values for $j^{th}$ feature and $l = \max_j|\Omega_j|$. Also, we set $y \in \{-1,1\}$. Finally, let $q\subseteq [m+1]$ be a subset of attributes, and $\Omega_q=\Pi_{j\in q} \Omega_{j}$.

\begin{definition}[Marginal of Dataset] \label{def:marginal}

The marginal of dataset $D$ on a subset of attributes $q$ is a vector $\vec{h}_q \in \mathbb{R}^{|\Omega_q|}$, indexed by domain element $\vec{t} \in \Omega_q$, such that each entry is  a count, i.e., $\vec{h}_q[\vec{t}] = \sum_{\vec{v} \in D}\mathbb{I}[\vec{v}[q]=\vec{t}]$. We let $M_q: V^n \rightarrow \mathbb{R}^{|\Omega_q|}$ denote the function that computes the marginal on $q$, i.e., $\vec{h}_q = M_q(D)$. 
\end{definition}

Given that a marginal is specified by an attribute set $q$, we also refer to $q$ as a marginal query. Moreover, for $d \leq m+1$, let $Q^m_{\leq d}$
consist of all
$q \subseteq [m+1]$ with size at most $d$. Furthermore, we say a set of marginals $\{\vec{h}_q\}_{q\in Q^m_{\leq d}}$ is \emph{consistent}, if there exists a dataset $D$, such that $M_q(D) = \vec{h}_q$ for all $q\in Q^m_{\leq d}$.

\subsection{Learning Linear Models with a Convex Loss}

We consider learning linear models for binary classification. Specifically, let $L(\vec{w},D)$ be the empirical risk of dataset $D$ on model $\vec{w}$ defined as $L(\vec{w},D) \triangleq \frac{1}{n}\sum_{(\vec{x},y)\in D} \varphi(\la\vec{w},\vec{x}\ra y))$, where $\varphi(\la\vec{w},\vec{x}\ra y)):\mathbb{R}\rightarrow\mathbb{R}$ is the loss of linear model $\vec{w}$ for sample $(\Vec{x},y)$. Throughout the paper, we consider $\varphi$ that is convex and Lipschitz.

\paragraph{Logistic Regression}
Logistic regression is a prominent representative model in learning linear models. We denote its empirical risk of dataset $D$ on model $\vec{w}$ as $\hat{L}(\vec{w},D) \triangleq \frac{1}{n}\sum_{(\vec{x},y)\in D}\hat{\varphi}(\la \vec{w},\vec{x}\ra y)$, where $ \hat{\varphi}(\la \vec{w},\vec{x}\ra y)= -\ln\left(\frac{1}{1+e^{-\la \vec{w},\vec{x}\ra y}}\right).$

% Similarly, 

\subsection{Polynomial Approximation}\label{sec:poly_approx}

Our proof of the upper bound relies on the technique of approximating the loss function with a bounded degree polynomial. Specifically, we consider the Bernstein polynomial \citep{bernstein1912proof,ROULIER1970117_bernstein, guan2009iterated}, which provides a theoretic analysis of its approximation error and the absolute values of its coefficients.

\begin{definition}[Bernstein Polynomial Approximation] Let $f$ be a function on $[a,b]$, the Bernstein polynomial approximation of degree $d$ is defined as
$$P_df(x) = \sum_{i=0}^d f\left(\frac{i}{d} \cdot (b-a)+a\right) \cdot B_{di}(x), a\leq x \leq b,$$
where $B_{di}(x) = \binom{d}{i} \cdot \left(\frac{x - a}{b - a}\right)^i \cdot \left(1 - \frac{x - a}{b - a}\right)^{d - i}$.
\end{definition}

Let $\|f\| = \max_{a\leq x \leq b} |f(x)|
$ denote the maximum absolute value when the function takes value from $[a,b]$.  
We utilize the following two error upper-bound of Bernstein polynomial approximations.

\begin{theorem}
[\citep{ROULIER1970117_bernstein}, Th. 1 and \citep{bernstein_r_0}, Th. 1.6.1]
\label{bernstein_1}
Suppose $a\leq 0<1\leq b$, and $f$ is a continuous function on $[a, b]$, for $d=1, 2, ...$,
\begin{align}
    \Norm{P_df-f}&\leq \frac{5}{4} \omega \left(f, \frac{b-a}{\sqrt{d}}\right)\label{bern_err},
\end{align}
where $\omega$ is the modulus of continuity of $f$ on $[a, b]$. Additionally, let $P_df(x)=\sum_{k=0}^d a_{dk}x^k$, then for $d=1, 2, ..., $
\begin{align}
    \sum_{k=0}^d |a_{dk}| &\leq \Norm{f}\left(1+\frac{2}{b-a}\right)^d.\label{bern_coef}
\end{align}

\end{theorem}

\begin{theorem}[\citep{bernstein_rth_order}]
\label{bernstein_2}
    Suppose $f$ is a function on $[0,1]$ with a continuous first-order derivative. For $d=1,2,\dots$,
    $$\Norm{P_df-f}\leq \frac{3}{4\sqrt{d}} \omega \left(f',\frac{1}{\sqrt{d}}\right),$$
    where $\omega$ is the modulus of continuity of $f'$, which is the first derivative of $f$.  
\end{theorem}

\subsection{Differential Privacy}
\ifsubmission
We adopt differential privacy~\cite{dwork2006} as our privacy notion, which has emerged as the dominant standard for managing the privacy risk to an individual associated with publicly sharing information about a dataset~\cite{10.1145/3219819.3226070, suriyakumar2020chasing, Fioretto_2021}. 
The central concept underlying differential privacy is that if the influence of any single substitution made to the dataset is sufficiently minimal, the resulting query outcome cannot be leveraged to infer much about any individual, thereby ensuring privacy. And we implement the Gaussian mechanism~\cite{dwork2014algorithmic} in Private-PGM to ensure the differential privacy.
%to enforce the differential privacy. 
See Appendix~\ref{app:DP_gaussian} for more details. 
\else
Differential privacy\cite{dwork2006} has emerged as the prevailing standard for managing the privacy risk to an individual associated with publicly sharing information about a dataset. We present the formal definition next.

\begin{definition}[$(\epsilon, \delta)$-Differential Privacy]
A randomized mechanism $\mathcal{M:D \rightarrow R}$ satisfies $(\epsilon, \delta)$-differential privacy if for any two adjacent inputs $x, x' \in \mathcal{D}$ and for any subset of outputs $\mathcal{S} \subseteq \mathcal{R}$  it holds that
$Pr[\mathcal{M}(x)\in \mathcal{S}]\leq e^{\epsilon} Pr[\mathcal{M}(x')\in \mathcal{S}]+\delta$.
\end{definition}

The
% \textbf{Laplace Mechanism}~\cite{dwork2006} and 
\textbf{Gaussian Mechanism}~\cite{dwork2014algorithmic} 
adds random noise drawn from a Gaussian distribution to a query output, where
the standard deviation of the noise is
%to the true output of a query, such that the amount of noise added is 
proportional to the sensitivity of the query.

\begin{theorem}[Gaussian Mechanism]\label{thm:gauss}
   Let $\epsilon\in(0,1)$ and $f:D\rightarrow R^d$, be an arbitrary d-dimensional function. Define its $l_2$ sensitivity to be $\Delta_2(f)=\max_{x,x'}\Norm{f(x)-f(x')}_2$, where $x, x'$ are any adjacent inputs in $\mathcal{D}$.  Let $
   \sigma^2=\frac{2 \Delta_2(f)^2\log{(1.25/\delta)}}{\epsilon^2}$. The Gaussian mechanism that adds noises sampled from $\mathcal{N}(0, \sigma^2)$ to each of the $d$ components of $f$'s output is $(\epsilon, \delta)$-differential privacy. 
\end{theorem}

Differential privacy is immune to post-processing~\cite{dwork2014algorithmic}: further computation on differentially private output will not further degrade the privacy guarantee.
\begin{theorem}[Post-Processing]\label{thm:post_proc}
    Let $\mathcal{M:D \rightarrow R}$ be a randomized algorithm that is $(\epsilon,\delta)$-differentially private. Let $f:\mathcal{R} \rightarrow \mathcal{R}'$ be arbitrary randomized mapping. Then $f \circ \mathcal{M}$ is $(\epsilon,\delta)$-differentially private.
\end{theorem}

\fi

\section{Upper Bound on the Excess Empirical Risk}\label{sec:upper_bound}

% \input{new_generic_theorem}
%In this section, 
We present our upper bound on the excess empirical risk for learning linear models with continuous and Lipschitz losses using synthetic data. %Specifically, 
In Section \ref{ssec:upperbound_loss_only}, we utilize the polynomial approximation techniques to show that the risk difference between the models trained from real and synthetic datasets can be bounded using the $\ell_1$ norm of marginal difference between real and synthetic datasets.
In Section \ref{ssec:upperbound_loss_priv}, we present an information-theoretic mechanism for generating synthetic data that is provably both marginal-preserving and DP, and we extend our theorems from Section \ref{ssec:upperbound_loss_only} to demonstrate a trade-off between privacy and loss.

Throughout this section, we let $D_r$ be the real dataset and $D_s$ be the synthetic dataset. We assume the datasets are normalized, i.e., for all $(\vec{x}, y) \in D_{r},D_{s}$ and for all $j\in[m]$, $\vec{x}[j] \in [-1,1]$. On the other hand, the label $y$ takes value from  $\{-1,1\}$,  and we may also refer to $y$ as the $(m+1)^{th}$ attribute. Given a set $q\in [m+1]$, let $\vec{h}_q^{(r)}$ and $\vec{h}_q^{(s)}$ denote the marginals of the real and synthetic datasets on $q$, i.e., $\vec{h}_q^{(r)} = M_q(D_r)$ and  $\vec{h}_q^{(s)} = M_q(D_s)$. Let $Q^m_{\leq d}$ be the set of all subsets of attributes (including label) with size no more than $d$.

\subsection{Bounding the Risk via Bounded Marginals' $\ell_1$-Distance }\label{ssec:upperbound_loss_only}

We begin by presenting a generic result, assuming only the loss function is continuous and Lipschitz.

\begin{theorem}\label{thm:ASD_gen}
Let $L(\vec{w}, D) = \sum_{(\vec{x},y)\in D} \frac{1}{n}\varphi(\la \vec{w},\vec{x} \ra y)$ such that $\varphi$ is continuous and $K$-Lipschitz. Let $\vec{w}_{r} = \mathsf{argmin}_{\vec{w},\|\vec{w}\|\leq \tau}L(\vec{w},D_r)$ and $\vec{w}_{s} = \mathsf{argmin}_{\vec{w},\|\vec{w}\|\leq \tau}L(\vec{w},D_s)$.
If for all $q \in Q^m_{\leq d}$, $\|\vec{h}_q^{(r)} - \vec{h}_q^{(s)}\|_1 \leq \nu$, then
\ifsubmissionshort{
\begin{align*}
    &|L(\vec{w}_{s},{D_{r}})-L(\vec{w}_{r},{D_{r}})|
\in O\left((K\cdot\tau\sqrt{m/(d-1)}\vphantom{+\frac{1}{n} \cdot (K\tau\sqrt{m}+ \varphi(0)) \cdot (3m \cdot \max\{1,\tau\})^{d-1} \nu}+\frac{1}{n} \cdot (K\tau\sqrt{m}+ \varphi(0)) \cdot (3m \cdot \max\{1,\tau\})^{d-1} \nu \right)
\end{align*}
}
\else{
$$ |L(\vec{w}_{s},{D_{r}})-L(\vec{w}_{r},{D_{r}})|\in O\left(K\cdot\tau\sqrt{m/(d-1)}+\frac{1}{n} \cdot (K\tau\sqrt{m}+ \varphi(0)) \cdot (3m \cdot \max\{1,\tau\})^{d-1} \nu \right).$$ }
\fi

\end{theorem}
Note that each sample in the dataset, $\vec{x}$, lies in the $m$-dimensional ball of radius $\sqrt{m}$.
If we set $\tau = \frac{1}{\sqrt{m}}$, we can view the optimization problem as consisting of datapoints contained in the $m$-dimensional unit ball and optimizing over linear models, $\vec{w}$, contained in the $m$-dimensional unit ball.
Thus, 
setting $\tau = \frac{1}{\sqrt{m}}$ and $K = O(1)$,
the above implies that as the size of the database $n$ goes to infinity, the excess empirical risk of the optimization problem is dominated by $O(\frac{1}{\sqrt{d-1}})$.

Our proof relies on the two generic upper bounds for any $\vec{w}$ such that$ \|\vec{w}\|_2 \leq \tau$. First, we construct an \emph{approximated empirical risk function} $L'$ through replacing the loss function $\varphi$ with its degree $d-1$ Bernstein polynomial approximation $P_{d-1}\varphi$. Then, we argue  $L'(\vec{w},D)\approx L(\vec{w},D)$ by invoking results on the maximum error in Bernstein polynomial approximation given in Theorem \ref{bernstein_1}. 

 Second, we bound the difference in empirical risk 
    $|L'(\vec{w}, D_s) - L'(\vec{w}, D_r)|$ between the real and synthetic datasets on \emph{any} linear model $\vec{w}$, by using the approximately marginal-preserving property of the synthetic dataset. Specifically, $P_{d-1}\varphi(\la\vec{w},\Vec{x}\ra y)$ can be expanded to a multivariate polynomial, where each monomial contains at most $d$ variables in $(\vec{x},y)$. Next, we can upper bound the risk $L'$, which is the average of this multivariate polynomial evaluated on each data sample, by a sum of the averages of individual monomials evaluated on each data sample. Then, this allows us to associate each average monomial with the marginal correponding to the set of attributes appearing in this monomial. Further, this average monomial value is fully determined given the corresponding marginal. Finally, we can apply the $\ell_1$ norm bound between the marginals of the real and synthetic datasets to bound the difference of each average monomial.

By applying the above bounds on different linear models in a sequence of inequalities, we arrive at the theorem statement.
We provide the formal proof in Appendix~\ref{app:ASD_gen}.

Next, we give a tighter bound for logistic regression. Our theorem can also extend to any loss function whose first derivative is continuous.

\begin{theorem}\label{thm:ASD_log}
Let $\hat{L}(\vec{w},D) = \frac{1}{n}\sum_{(\vec{x},y) \in D}\hat{\varphi}(\la \vec{w},\vec{x}\ra y)$. Let $\vec{w}_{r} = \mathsf{argmin}_{\vec{w},\|\vec{w}\|\leq \tau}\hat{L}(\vec{w},D_r)$ and $\vec{w}_{s} = \mathsf{argmin}_{\vec{w},\|\vec{w}\|\leq \tau}\hat{L}(\vec{w},D_s)$.
% , and $\tau = \max(\|\vec{w}_{r}\|_2,\|\vec{w}_{s}\|_2)$.
If for all $q \in Q^m_{\leq d}$, $\|\vec{h}_q^{(r)} - \vec{h}_q^{(s)}\|_1 \leq \nu$, then
\ifsubmissionshort{
$$ |\hat{L}(\vec{w}_{s},{D_{r}})-\hat{L}(\vec{w}_{r},{D_{r}})|\in O\left(\tau\sqrt{m}/(d-1)+\frac{1}{n} \cdot \tau\sqrt{m}\cdot (2m \cdot \max\{1,\tau\})^{d-1} \nu \}\right).$$
}
\else{
$$ |\hat{L}(\vec{w}_{s},{D_{r}})-\hat{L}(\vec{w}_{r},{D_{r}})|\in O\left(\tau\sqrt{m}/(d-1)+\frac{1}{n} \cdot \tau\sqrt{m}\cdot (2m \cdot \max\{1,\tau\})^{d-1} \nu \}\right).$$ }
\fi
\end{theorem}

Setting $\tau = \frac{1}{\sqrt{m}}$,
the above implies that as the size of the database $n$ goes to infinity, the excess empirical risk is dominated by $O(\frac{1}{d-1})$ for logistic regression.

We provide the formal proof in Appendix~\ref{app:ASD_log}, wherein the primary difference is we apply a tighter Bernstein polynomial approximation bound from Theorem~\ref{bernstein_2}.

\textbf{Polynomial Approximation Error}
The term $\tau \sqrt{m}/(d-1)$ in the bound in Theorem~\ref{thm:ASD_log} comes from the error of the degree-$(d-1)$ Bernstein polynomial approximating the $\log(\text{sigmoid}(\cdot))$ function on the interval $[-\tau \cdot \sqrt{m}, \tau \cdot \sqrt{m}]$.
For a fixed degree $d$, the Bernstein polynomial approximation may not yield the best error. Replacing it with an approximation with better error immediately leads to an improvement in the upper bound.
We therefore investigate two alternative methods for polynomial approximation, namely the minimax
approximation \cite{davis1975interpolation} and an ``iterated'' Bernstein approximation \cite{bernstein1912proof,ROULIER1970117_bernstein, guan2009iterated}. 
Refer to Figure \ref{fig:approx} below for examples illustrating the quality of the approximations of the $\log(\text{sigmoid(x)})$ function by 4-degree polynomial functions obtained by using the minimax and iterated Bernstein approximations. 

\begin{figure} [tb]
% \vskip 0.1in
\hspace{-11em}
\centering
\begin{minipage}{.25\textwidth}
\begin{subfigure}[Minimax Approximation]
{\includegraphics[scale=0.3]{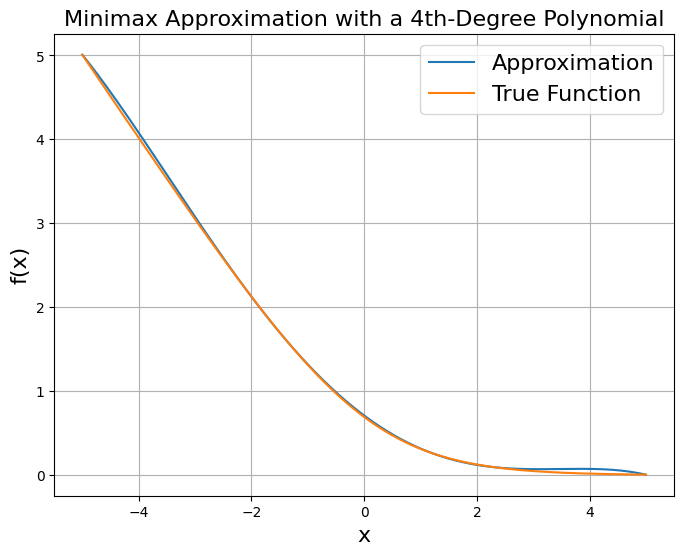}}
\label{fig:minimax}
\end{subfigure}
\end{minipage}
\hspace{12em}
\begin{minipage}{.25\textwidth}
\begin{subfigure}[Bernstein Approximation]
{\includegraphics[scale=0.3]{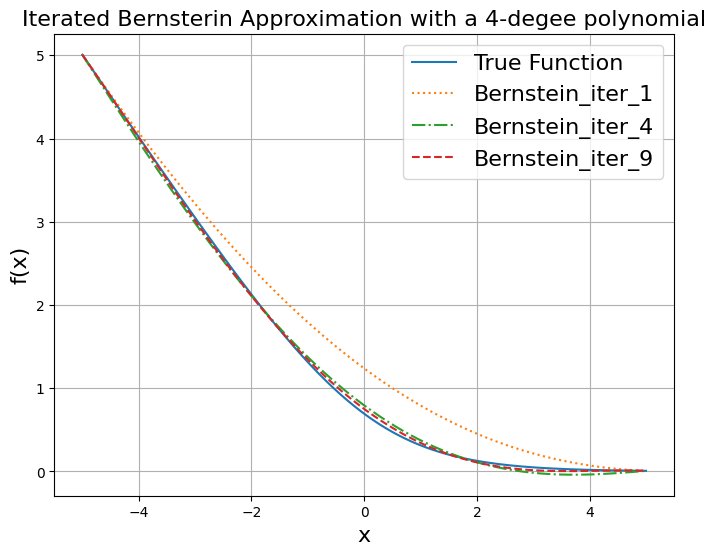}}
\label{fig:bernstein}
\end{subfigure}
\end{minipage}
\vskip 0.1in
\caption{(a) shows Minimax approximation for $\log(\text{sigmoid(x)})$ function within interval $[-5, 5]$ in 4-degree polynomial: $\log(\text{sigmoid(x)})_{minimax}\approx 0.71-0.5x+ 0.1096 x^2-0.0015 x^4 $, with an error of 0.061.
(b) shows the iterated Bernstein Approximations for $\log(\text{sigmoid(x)})$ function within interval $[-5, 5]$ in 4-degree polynomial by iterate Bernstein approximation for 1 time, 4 times, and 9 times: $\log(\text{sigmoid(x)})_{Bern_1}\approx 1.2377- 0.5x+0.0544x^2-0.0001x^4 $, with an error of 0.545;
$\log(\text{sigmoid(x)})_{Bern_4}\approx 0.7934- 0.5x+0.0812x^2-0.0005x^4 $, with an error of 0.100;
$\log(\text{sigmoid(x)})_{Bern_9}\approx 0.7504- 0.5x+0.0931x^2-0.0009x^4 $, with an error of 0.057.}
% \vskip 0.1in
\label{fig:approx}
\end{figure}

Through observations, three significant findings emerge. Firstly, the approximation error reduces while the polynomial degree increases for both approximation methods. Secondly, the error reduces with each successive iteration of the Bernstein approximation. Thirdly, the 9th-iterated Bernstein polynomial approximation slightly outperforms the minimax polynomial approximation in our experimental results. Nevertheless, we opt for Bernstein Approximation in our subsequent analysis, which provides a theoretic analysis of its approximation error and the absolute values of its coefficients.
However, any polynomial approximation method can be used interchangeably in practical applications or in our analysis without losing generality by simply switching its approximation error bound according to the approximation method would be employed. 

\subsection{DP and marginal-preserving synthetic data}
\label{ssec:upperbound_loss_priv}

%In this subsection,

We present a DP synthetic data generating mechanism that preserves $\ell_1$ norm of all marginals with order no more than $d$ (with overwhelming probability), and analyze the end-to-end privacy and utility trade-off. 

\ifsubmissionshort{
\begingroup
\floatname{algorithm}{Mechanism}
\begin{algorithm}
   \caption{Generating Synthetic Data $\mathsf{Gen}_{d, \sigma}$}
    \label{alg:syn_data}
\begin{algorithmic}
   \STATE {\bfseries Input:} Real dataset $D_{r}$, number of samples $n$
  \STATE {\bfseries Output:} Synthetic Dataset $D_{s}$
    \STATE Measure Noise Marginals $\mathsf{Mea}(D_r)$:
   \FOR{$q \in Q^m_{\leq d}$}
   \STATE Measuring marginal: $\vec{h}_q^{(r)}=M_q(D_{r})$;
   \STATE Add noise: $\vec{\hat{h}}_q \leftarrow \vec{h}_q^{(r)}+\mathcal{N}(0, \sigma)$;
   \ENDFOR
    \STATE {\bfseries Generate synthetic data} $\mathsf{Syn}(n,\{\vec{\hat{h}}_q\}_{q \in Q^m_{\leq d}})$:
    \STATE (Brute Force) Find $D_s$ that minimizes the maximum $\ell_1$ difference with respect to marginals in $\{\vec{\hat{h}}_q\}_{q \in Q^m_{\leq d}}$, i.e, $D_s = \mathsf{argmin}_{D} \max_{q \in Q^m_{\leq d}} \|\vec{\hat{h}}_q - M_q(D)\|_1$.    
\end{algorithmic}
\end{algorithm}
\endgroup
}
\else{
\begin{algorithm}
\caption{Generating Synthetic Data $\mathsf{Gen}_{d, \sigma}$}
\label{alg:syn_data}
\DontPrintSemicolon
  \KwIn{Real dataset $D_{r}$, number of samples $n$}
  \KwOut{Synthetic Dataset $D_{s}$}

\BlankLine
  \textbf{Measure Noise Marginals $\mathsf{Mea}(D_r)$:} \;

  \For{$q \in Q^m_{\leq d}$}{Measuring marginal: $\vec{h}_q^{(r)}=M_q(D_{r})$\;
 Add noise: $\vec{\hat{h}}_q \leftarrow \vec{h}_q^{(r)}+\mathcal{N}(0, \sigma)$.\;}
\BlankLine
\textbf{Generate Synthetic Data $\mathsf{Syn}(n,\{\vec{\hat{h}}_q\}_{q \in Q^m_{\leq d}})$:}\;
     (Brute Force) Find $D_s$ that minimizes the maximum $\ell_1$ difference with respect to marginals in $\{\vec{\hat{h}}_q\}_{q \in Q^m_{\leq d}}$, i.e, $D_s = \mathsf{argmin}_{D} \max_{q \in Q^m_{\leq d}} \|\vec{\hat{h}}_q - M_q(D)\|_1$.    
\end{algorithm}
}
\fi

The differential privacy guarantee of Mechanism~\ref{alg:syn_data} follows directly from Theorem~\ref{thm:gauss} and Theorem~\ref{thm:post_proc}.
\begin{lemma}\label{lem:syn_dp}
    $\mathsf{Gen}_{d, \sigma}$ is $(\epsilon,\delta)$-DP, if $\sigma=\frac{2 m^{d/2}\sqrt{\ln(1.25/\delta)}}{\epsilon}$.
\end{lemma}

We provide the formal proof in Appendix~\ref{app:syn_dp}.

Next, we bound the $\ell_1$ difference between noisy and real marginals using Chernoff bound.

\begin{lemma}\label{lem:syn_bound}
Let $D_s \leftarrow \mathsf{Gen}_{d,\sigma}$. Then
$\|\vec{h}_q^{(r)} - M_q(D_s)\|_1 \leq 2l^d \sqrt{2(\ln(2)(1+\lambda)+d\ln(ml))}\sigma$ for all $q \in Q^m_{\leq d}$ with $1-2^{-\lambda}$ probability.
\end{lemma}

We provide the formal proof in Appendix~\ref{app:syn_bound}.

Using Lemmas \ref{lem:syn_dp} and \ref{lem:syn_bound} allows us to represent the marginal difference $\nu$ in previous Theorems \ref{thm:ASD_gen} and \ref{thm:ASD_log} with the expression containing the privacy parameters. In particular, it yields the following corollaries.

\begin{corollary}
    Let $L(\vec{w};(\vec{x},y)) = \varphi(\la \vec{w},\vec{x} \ra y)$ such that $\varphi$ is continuous and $K$-Lipschitz. 
    Let $D_s \leftarrow \mathsf{Gen}_{d, \sigma}(D_r)$, where $\sigma =\frac{2 m^{d/2}\sqrt{\ln(1.25/\delta)}}{\epsilon}$.
    Then $D_s$ satisfies $(\epsilon,\delta)$-DP.
    Additionally, let $\vec{w}_{r} = \mathsf{argmin}_{\vec{w},\|\vec{w}\|\leq \tau}L(\vec{w},D_r)$ and $\vec{w}_{s} = \mathsf{argmin}_{\vec{w},\|\vec{w}\|\leq \tau}L(\vec{w},D_s)$.
    % , and $\tau = \max(\|\vec{w}_{r}\|_2,\|\vec{w}_{s}\|_2)$.
Then
\ifsubmissionshort{
\begin{align*}
    &|L(\vec{w}_{s},{D_{r}})-L(\vec{w}_{r},{D_{r}})|\\
    &\in O\left(K\cdot\tau\sqrt{m/(d-1)}+\frac{1}{n} \cdot (K\tau\sqrt{m}+ \varphi(0)) \cdot (3m \cdot \max\{1,\tau\})^{d-1} \cdot2l^d \sqrt{2(\ln(2)(1+\lambda)+d\ln(ml))} \cdot \sigma \right),
\end{align*}
    
}
\else{
\begin{align*}
    &|L(\vec{w}_{s},{D_{r}})-L(\vec{w}_{r},{D_{r}})|\\
    &\in O\left(K\cdot\tau\sqrt{m/(d-1)}+\frac{1}{n} \cdot (K\tau\sqrt{m}+ \varphi(0)) \cdot (3m \cdot \max\{1,\tau\})^{d-1} 2l^d \sqrt{\frac{2(\lambda +1+d \lg(ml))}{\lg e}} \sigma \right),
\end{align*}
}
\fi

except with $2^{-\lambda}$ probability.
\end{corollary}

\begin{corollary}
    Let $D_s \leftarrow \mathsf{Gen}_{d, \sigma}(D_r)$, where $\sigma =\frac{2 m^{d/2}\sqrt{\ln(1.25/\delta)}}{\epsilon}$.
    Then $D_s$ satisfies $(\epsilon,\delta)$-DP.
    Additionally, let $\vec{w}_{r} = \mathsf{argmin}_{\vec{w},\|\vec{w}\|\leq \tau}\hat{L}(\vec{w},D_r)$ and $\vec{w}_{s} = \mathsf{argmin}_{\vec{w},\|\vec{w}\|\leq \tau}\hat{L}(\vec{w},D_s)$.
Then
\ifsubmissionshort{

    \begin{align*}
    &|\hat{L}(\vec{w}_{s},{D_{r}})-\hat{L}(\vec{w}_{r},{D_{r}})|\in O\left(\tau\sqrt{m}/(d-1)+\frac{1}{n} \cdot \tau\sqrt{m}\cdot (3m \cdot \max\{1,\tau\})^{d-1} \cdot 2l^d \sqrt{2(\ln(2)(1+\lambda)+d\ln(ml))}\cdot \sigma \right),
    \end{align*}}
\else{
\begin{align*}
    &|\hat{L}(\vec{w}_{s},{D_{r}})-\hat{L}(\vec{w}_{r},{D_{r}})|\\
    &\in O\left(\tau\sqrt{m}/(d-1)+\frac{1}{n} \cdot \tau\sqrt{m}\cdot (3m \cdot \max\{1,\tau\})^{d-1} 2l^d \sqrt{\frac{2(\lambda +1+d \lg(ml))}{\lg e}} \sigma \right),
\end{align*}}
\fi
except with $2^{-\lambda}$ probability.
\end{corollary}

As in the previous section, setting $\tau = \frac{1}{\sqrt{m}}$ and $K = O(1)$,
the above corollaries imply that as the size of the database $n$ goes to infinity, the excess empirical risk is dominated by $O(\frac{1}{\sqrt{d-1}})$ for general continuous and $K$-Lipschitz loss functions, and dominated by $O(\frac{1}{d-1})$ for logistic regression.

\section{Lower Bound on the Excess Empirical Risk}\label{sec:lower_bound}
We next present a theorem that shows that our upper bound in the previous section is nearly tight for certain ranges of parameter settings.
Specifically, we show that there exists a 
distribution over datasets $D_r$, a convex, 2-Lipschitz cost function $L$, and a
range of parameter settings for
$n, m, d, \tau$
such that 
Theorem~\ref{thm:ASD_gen}
implies the existence of a synthetic data generation algorithm with excess risk at most
$O\left (\sqrt{\frac{\ln(\ln(n))}{\ln(n)}} \right) + O\left (\frac{\ln(n)}{n^{1/4}} \right)$.
%\]
On the other hand, we show that
for any synthetic data generation algorithm
$\mathsf{Syn}$
(of a particular form), the excess risk is at least
$\Omega\left (\frac{1}{\ln^3(n)} \right )$.
%\]
Thus, both the upper and lower bounds on the difference in loss are fixed polynomials in $\frac{1}{\ln(n)}$, where $n$ is the size of the dataset.
Although existing differentially private convex optimization methods such as gradient perturbation\citep{6979031, yu2021gradient}, output perturbation\citep{zhang2017efficient}, and objective perturbation \citep{chaudhuri2011differentially} demonstrate an error of $O(1/n)$ or less, it is crucial to highlight the primary advantages of synthetic data: the ability to execute numerous downstream tasks without compromising the privacy guarantee, along with the flexibility to employ any non-private learning algorithm out-of-the-box.

Our lower bound captures synthetic data generation algorithms that obtain noisy marginals as input and then use an arbitrary (potentially computationally unbounded), randomized algorithm to construct a synthetic dataset from these noisy marginals.
The synthetic data generation algorithms that we consider may not make any assumption about the distribution of the inputted noisy marginals, other than the fact that \emph{each noisy marginal is close (within some tolerance) to the true expectation of the data distribution}.
Our matching upper bound holds for synthetic data generation 
of this form, since Theorem~\ref{thm:ASD_gen} does not make any distributional assumption on $\vec{h}_q^{(s)}$ but only requires that $\|\vec{h}_q^{(s)} - \vec{h}_q^{(r)}\|_1 \leq \nu$.

Before presenting our Theorem and proof, we begin with some notation.
For a vector $\vec{v}=(v[j])_{j\in q}$, let $n\cdot\vec{v}=(n\cdot v[j])_{j\in q}$.
Let $\mathcal{D}_m$ be a distribution over $(\vec{x},y)$, where
input $\vec{x} \in \{-1, 1\}^m$ and label $y \in \{-1,1\}$. In our writeup, we treat $(\vec{x},y)$ as a single vector where the last entry is $y$.
We say that a set of vectors
$\{\vec{u}_q\}_{q \in Q^m_{\leq d}}$
has tolerance $\mathsf{tol}$ relative to distribution $\mathcal{D}_m$ if 
$\forall q \in Q^m_{\leq d}$, $\forall \vec{t} \in \Omega_q$, $\|\vec{u}_q(\vec{t}) -\mathbb{E}_{(\vec{x},y) \sim \mathcal{D}_m}[\mathbb{I}[(\vec{x},y)[q]=\vec{t}]]\|_\infty \leq \mathsf{tol}$, where $\mathbb{I}[(\vec{x},y)[q] = \vec{t}]$ is the indicator variable
set to $1$ if $(\vec{x},y)[q] = \vec{t}$ and set to $0$ otherwise.
Let $\mathsf{Syn}$ be a synthetic data generation algorithm that receives as input $n \in \mathbb{N}$ and a set of vectors
$\{n \cdot\vec{u}_q\}_{q \in Q^m_{\leq d}}$
and uses it in an arbitrary way to output a synthetic database $D_s$ of size $n$
with marginals $\{\vec{h}_q^{s}\}_{q \in Q^m_{\leq d}}$.

\begin{theorem} \label{th:lower_bound}
For sufficiently large $n$,
$m = m(n) = O(\ln^6(n))$
and $d = d(n) = O(\frac{\ln(n)}{\ln \ln (n)})$, 
there exists a cost function $L(\vec{w},D) := \frac{1}{n} \sum_{(\vec{x},y) \in D} \varphi(\langle \vec{w}, \vec{x} \rangle \cdot y)$ with $\varphi$ being
$\frac{1}{\ln^3(n)}$-strongly convex and 2-Lipschitz,
for which
the following hold:
\ifsubmissionshort{
\begin{itemize}
\item There exists a deterministic algorithm $\mathsf{Syn}$
such that
for all distributions $\mathcal{D}_m$ and all sets of
vectors $\{\vec{u}_q\}_{q \in Q^m_{\leq d}}$
with tolerance $\mathsf{tol} = \frac{1}{n}$ relative to $\mathcal{D}_m$,
with all but negligible probability over 
$D_r \sim \mathcal{D}_m^n$,

\begin{align*}
    |L(\vec{w}_r,D_r) -   L(\vec{w}_s,D_r)| \in O\left (\sqrt{\frac{\ln(\ln(n))}{\ln(n)}} \right) + O\left (\frac{\ln(n)}{n^{1/4}} \right),
\end{align*}

\item For every randomized algorithm $\mathsf{Syn}$,
there exists a distribution $\mathcal{D}_m$ and a set of vectors $\{\vec{u}_q\}_{q \in Q^m_{\leq d}}$
with tolerance $\mathsf{tol} = \frac{1}{n}$ relative to $\mathcal{D}_m$,
such that
with all but negligible probability over $D_r \sim \mathcal{D}_m^n$, 

\begin{align*}
\left | L(\vec{w}_r,D_r) -   \mathbb{E}_{D_s \leftarrow \mathsf{Syn}(n, \{n\cdot\vec{u}_q\}_{q \in Q^m_{\leq d}})}[L(\vec{w}_s,D_r)] \right | \in \Omega\left (\frac{1}{\ln^3(n)} \right ),
\end{align*}

\end{itemize}
}
\else{
\begin{itemize}
\item There exists a (computationally inefficient) deterministic algorithm $\mathsf{Syn}$
such that
for all distributions $\mathcal{D}_m$ and all sets of
vectors $\{\vec{u}_q\}_{q \in Q^m_{\leq d}}$
with tolerance $\mathsf{tol} = \frac{1}{n}$ relative to $\mathcal{D}_m$,
with all but negligible probability over 
$D_r \sim \mathcal{D}_m^n$,
and for $D_s = \mathsf{Syn}(n, \{n\cdot\vec{u}_q\}_{q \in Q^m_{\leq d}})$,
\begin{align*}
   \left | L(\vec{w}_r,D_r) -   L(\vec{w}_s,D_r) \right |
\in O\left (\sqrt{\frac{\ln(\ln(n))}{\ln(n)}} \right) + O\left (\frac{\ln(n)}{n^{1/4}} \right), 
\end{align*}
\item For every (even computationally inefficient) randomized algorithm $\mathsf{Syn}$,
there exists a distribution $\mathcal{D}_m$ and a set of vectors $\{\vec{u}_q\}_{q \in Q^m_{\leq d}}$
with tolerance $\mathsf{tol} = \frac{1}{n}$ relative to $\mathcal{D}_m$,
such that
with all but negligible probability over $D_r \sim \mathcal{D}_m^n$, 
\begin{align*}
\left | L(\vec{w}_r,D_r) -   \mathbb{E}_{D_s \leftarrow \mathsf{Syn}(n, \{n\cdot\vec{u}_q\}_{q \in Q^m_{\leq d}})}[L(\vec{w}_s,D_r)] \right |
\in \Omega\left (\frac{1}{\ln^3(n)} \right ),
\end{align*}
\end{itemize}
}\fi
where
$\vec{w}_s = \mathsf{argmin}_{\vec{w}} L(\vec{w},D_s)$,
and
$\vec{w}_r = \mathsf{argmin}_{\vec{w}} L(\vec{w},D_r)$.
\end{theorem}

Our main insight to achieve the above result is that
for any $\{\vec{u}_q\}_{q \in Q^m_{\leq d}}$ with tolerance $\mathsf{tol} = \frac{1}{n}$, and any 
algorithm $\mathsf{Syn}$, e.g., the subroutine used in Mechanism \ref{alg:syn_data},
the Algorithm~\ref{fig:non_adap_stat_query_alg}
can be viewed as a \emph{non-adaptive statistical query}
learning algorithm 
 that makes
$\sum_{q\in Q^m_{\leq d}} |\Omega_q| \leq  m^d \cdot 2^d$ number of statistical
queries.

\ifsubmissionshort{
\begin{algorithm}
   \caption{A non-adaptive statistical query algorithm.}
\label{fig:non_adap_stat_query_alg}
\begin{algorithmic}
   \STATE {\bfseries Let} $\{\vec{u}_q\}_{q \in Q^m_{\leq d}}$
%\sim \mathcal{H}^{\mathcal{D}_m}(Q^m_{\leq d})$.
represent the responses of a statistical query oracle on the non-adaptive queries $\vec{t} \in \Omega_q$, for every $ q\in Q^m_{\leq d}$;
   \STATE {\bfseries Set} $D_s \leftarrow \mathsf{Syn}(n, \{n\cdot\vec{u}_q\}_{q \in Q^m_{\leq d}})$;
   \STATE {\bfseries Output} $\vec{w}_s = \mathsf{argmin}_{\vec{w}} L(\vec{w},D_s)$;
\end{algorithmic}
\end{algorithm}
}
\else{
\begin{algorithm}
\begin{enumerate}
\item Let $\{\vec{u}_q\}_{q \in Q^m_{\leq d}}$
%\sim \mathcal{H}^{\mathcal{D}_m}(Q^m_{\leq d})$.
represent the responses of a statistical query oracle on the non-adaptive queries $\vec{t} \in \Omega_q$, for every $ q\in Q^m_{\leq d}$.
\item Set $D_s \leftarrow \mathsf{Syn}(n, \{n\cdot\vec{u}_q\}_{q \in Q^m_{\leq d}})$
\item Output $\vec{w}_s = \mathsf{argmin}_{\vec{w}} L(\vec{w},D_s)$
\end{enumerate}
\caption{A non-adaptive statistical query algorithm.}
\label{fig:non_adap_stat_query_alg}
\end{algorithm}
}
\fi

We can  \citep{DF20}.
Their Theorem 5 holds
even for \emph{large-margin linear separators}, where the target concept class consists of linear separators $\vec{w}$ such that for every
$(\vec{x},y)$ in the support of $\mathcal{D}$,
$\frac{\langle \vec{x}, \vec{w} \rangle y}{|\vec{x}|\cdot |\vec{w}|} \geq \gamma$.
As we will see later, this large-margin will allow us to convert the lower bound given in Theorem~\ref{th:stat_query_lb} which shows a gap in \emph{accuracy}, to a result which shows a gap in \emph{cost} (for cost function $L_{D_r}$) between the optimal linear separator and the linear separator outputted by the non-adaptive statistical query algorithm.
We provide the formal proof of Theorem~\ref{th:lower_bound} in Appendix~\ref{app:theorem 4.1 proof}.

\section{Experimental Evaluation}\label{sec:exp}
% \input{experiment_full}
%In this section, 
We conducted experiments\footnote{Our experiments code and datasets are available at \url{https://github.com/DPML-syn/MarginalPreserving_DP_SyntheticData}} to evaluate the performance of DP and marginal-preserving synthetic data generation on six public datasets.  We select AIM~\citep{aim}, 
the typical and notable mechanism from among the marginal-preserving methods, to generate the DP synthetic data. 
The assessment utilizes the ``Train on Synthetic, Test on Real" (TSTR) approach~\citep{TSRS}, where we train the real-data-model and synthetic-data-model (using the scikit-learn's\citep{scikit-learn} library of logistic regression with the LBFGS solver), and evaluate both models on the real test data. Furthermore, we employ two other widely recognized DPML methods, DP-SGD~\citep{Abadi_2016} and PATE learning~\citep{papernot2017semisupervised}, for comparison with our proposed marginal-preserving synthetic data training.

\subsection{Dataset} \label{sec: dataset}
For our experimental evaluation, we utilized six datasets along with data preprocessing: Adult\citep{adult}, Compas\citep{compas}, Churn\citep{churn}, Dutch\citep{dutch}, Law\citep{law} and Heart\citep{heart}, refer to the Table \ref{tbl: summary of dataset} for an overview of the datasets. 
\ifsubmissionshort{

\begin{table} [H]
\centering
\vskip 0.1in
\caption{Summary of datasets used in experiments}
\vskip 0.1in
\label{tbl: summary of dataset}
\begin{adjustbox}{center}
% \begin{tabular}{llllll}
% \toprule
% Dataset & Size & \#Dim &  Dataset & Size & \#Dim \\
% \midrule
% Adult& 48,842&14  &
% Compas& 7,214&9 \\
% Churn& 3,859& 16 &
% Heart &303 & 14\\
% Law &20,798 & 12 &
% Dutch & 60,420& 12\\
% \bottomrule     
% \end{tabular}
% \end{adjustbox}
% \vskip -0.1in
% \end{table}

% \begin{table} 
% \centering
% \caption{Summary of datasets used in experiments}
% \label{tbl: summary of dataset}
% \begin{adjustbox}{center}
\begin{tabular}{llllllllllll} 
\toprule
Dataset & Adult & Compas&Churn & Heart &Law &Dutch \\
\midrule
Instances &48,842 & 7,214& 3,859& 303& 20,798 & 60,420\\
Features & 14 & 9&16&14&12 &12\\
\bottomrule     
\end{tabular}
\end{adjustbox}
\end{table}
}
\else{
\begin{table} 
\centering
\caption{Summary of datasets used in experiments}
\label{tbl: summary of dataset}
\begin{adjustbox}{center}
\begin{tabular}{llllllllllll}
\toprule
Dataset & Adult & Compas&Churn & Heart &Law &Dutch \\
\midrule
Instances &48,842 & 7,214& 3,859& 303& 20,798 & 60,420\\
Features & 14 & 9&16&14&12 &12\\
\bottomrule     
\end{tabular}
\end{adjustbox}
\end{table}
}
\fi
\subsection{Synthetic Data Generation}
Marginal-based approaches are the state-of-art method for preserving key statistical properties of the ground truth data to generate synthetic data with DP guarantees. 
In our experiments, we examined its performance in DP-ML setting.
The marginal-based approaches all align with select-measure-generate framework, which, at a high level, can be divided into three steps:
(1) Select sets of attributes, referred to as marginal queries, each containing at most $d$ attributes; (2) Using the real dataset, compute the marginal for each selected query, with injected noise;
(3) Generate synthetic data that matches the noisy marginals as closely as possible.
We opt for one of the leading marginal-based mechanisms, AIM~\citep{aim}, 
to validate the effectiveness of marginal preserving synthetic data. 
AIM is built using the core component Private-PGM~\citep{Graphical-model}, wherein, Private-PGM operates for steps 2 and 3 in the framework. Additionally, AIM incorporates a greedily and iteratively algorithm to fulfill step 1.
We defer a more detailed discussion of Private-PGM, AIM to Appendix \ref{app:syn_gen}.

\subsection{Data Preprocessing}
The raw data we use to generate synthetic data may present various challenges, including missing values or containing continuous values that require conversion to discrete numbers.
Therefore, we executed a series of data preprocessing before inputting it to the synthetic data generation mechanism: 
(1). Cleaning noisy data, by e.g.~deleting data samples that contained missing values.
(2). Converting categorical variables like gender and nationality into numerical values,  to make them suitable for machine learning algorithms. (3). Converting continuous variables, such as income, into discrete values,
while preserving the original ascending order of values.  The quantization method employed here 
is a simple bucketing approach. More sophisticated quantization methods, such as those discussed in Gersho et al.~\citep{quantization_book}, could lead to improved handling of continuous data.
(4). Feature scaling, to scale numeric features to a standard range starting from 0.

We highlight that these pre-processing steps applied to the real data do not compromise the privacy guarantee of the outputted synthetic data, since the data-preprocessing steps do not impact the sensitivity of the marginals, which determines the amount of noise added. 
Leveraging the post-processing theorem~\ref{thm:post_proc}, we can safely perform any supplementary data-preprocessing steps, e.g.~data normalization, before engaging on subsequent ML training. This augments the model's training effectiveness without degrading its privacy. 

\subsection{Evaluation Metrics}
We evaluate the performance of our approach using metrics of accuracy and ROC-AUC score, as they are commonly used and provide a comprehensive evaluation of classification performance. Accuracy measures the proportion of correctly classified samples, while ROC-AUC score indicates how well the classifier discriminates between the positive and negative classes. Additionally, we also compare the empirical risk from both models on real testing data.

\ifsubmissionshort{
To gain better insight of marginal-preserving synthetic data, we conducted comparative experiments: We generate synthetic data for six(6) dataset with each eight(8) DP parameters, $\epsilon$. 
We compared the performance of trained ML model in these synthetic data across with different $\epsilon$. 
Furthermore, as a supplementary investigation, we conducted two additional experiments: (1). We compare marginal-preserving synthetic data approach with current predominant Training-DPML approaches: PATE learning and DP-SGD (refer to Appendix~\ref{app: Comparison with Other DPML Techniques} for details). 
(2) We demonstrated the model-agnostic advantage of AIM's synthetic data by evaluating its synthetic data on training in two classifiers with distinct target labels. AIM proves to be effective without requiring prior knowledge of which features specifically correspond to the downstream classification task, and consistently maintains its performance across diverse classifiers (see in Appendix~\ref{app: Assess AIM for Different Classifiers}). This is desirable in the synthetic data setting, since the goal is to generate synthetic data once, and subsequently train many models on the same synthetic data.
}
\else{
To gain an better insight of marginal-preserving synthetic data, we conducted three comparative experiments: (1). We generate synthetic data for each dataset with eight(8) DP parameters, $\epsilon$. 
We will compare the performance of trained ML model in these synthetic data across with different privacy budgets. 
(2). We conduct a comparison of marginal-based synthetic data approach with current predominant Training-DPML approaches: PATE learning and DP-GD.  (3). 
Additional, to demonstrate the model-agnostic advantage of AIM's synthetic data, we evaluate its synthetic data on training in two classifiers with distinct target labels, AIM proves to be effective without requiring prior knowledge of which features specifically correspond to the downstream classification task, and consistently maintains its performance across diverse classifiers. This is desirable in the synthetic data setting, since the goal is to generate synthetic data once, and subsequently train many models on the same synthetic data.
}\fi

\subsection{Results} 
\ifsubmissionshort{}\else{
\subsubsection{Performance of Synthetic data with various privacy budgets}}\fi
\label{sec:experiment Synthetic data in varied privacy budgets}
We assessed the generated synthetic datasets on (1) how well they preserved the marginals and (2) the performance of ML model training on the synthetic data. 
We utilize the normalized-$\ell_1$ errors to evaluate the effectiveness of marginals preservation for different synthetic datasets generated. Here, the normalized-$\ell_1$ error for a marginal query $q \in Q_{\leq d}^m$, 
\begin{minipage}{.4\textwidth}
    is defined as $\frac{||\vec{h}^{(r)}_q - \vec{h}^{{(s)}}_q||_1}{n}$, where
$\vec{h}^{(r)}_q$ and $\vec{h}^{(s)}_q$ denote the marginals of the real and synthetic datasets on $q$, and $n$ is the size of the real dataset.  AIM mechanism reports an average normalized-$\ell_1$ error over all selected marginal queries, with the assertion that these errors serve as upper bounds for the maximum error across all marginals with at most $d$ attributes, (including both selected and non-selected ones.) This assertion is substantiated by Theorem~\ref{thm:AIM bound non-selected marginal}.
In our experiment, we set $d=4$.
The computed normalized-$\ell_1$ errors are shown in Figure~\ref{fig:l1_error}. It is easy to see that the higher the privacy budget, the less noise added into marginal measurements and so the smaller the normalized-$\ell_1$ error in synthetic data. 
\end{minipage}
\hfill
\begin{minipage}{.55\textwidth}
    \begin{figure}[H]
    % \vskip 0.1in
        \centering
\includegraphics[scale=0.3]{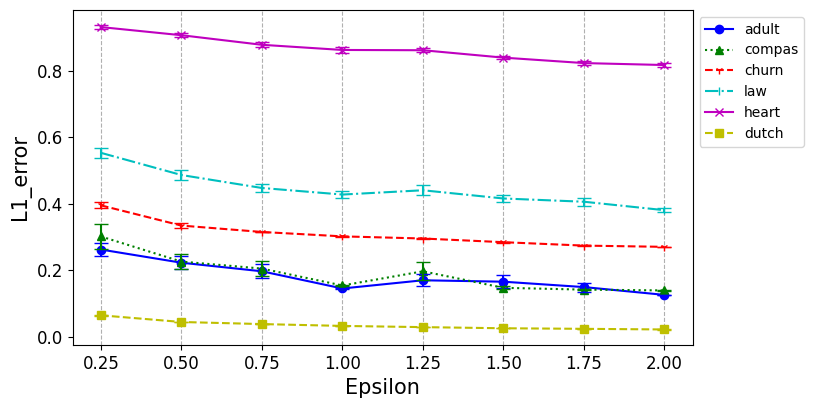}
% \vskip 0.1in
\caption{We compare the  $L_1$ error of synthetic data using AIM mechanism for all six(6) datasets with different privacy budget.}
\label{fig:l1_error}
\vskip 0.1in
\end{figure}

\end{minipage}

In Figure~\ref{fig:performance_in_diff_eps} 
\ifsubmissionshort{(see full results in Appendix~\ref{app: Performance of synthetic data with various privacy budgets} Table~\ref{tab:performance_in_diff_eps},)}
\else{and Table~\ref{tab:performance_in_diff_eps}}
\fi we present our empirical results on the performance of ML models that are trained using marginal preserving synthetic datasets. The results show that the models acquired from training on the synthetic datasets with higher privacy budget exhibit higher accuracy, 
% ROC AUC score, 
and lower excess empirical risk. In conjunction,
%with the previous observation, 
we note that higher privacy budget enables us to achieve smaller $\ell_1$ error synthetic data, leading to better synthetic data performance on ML training.  
Moreover, we observe that among all synthetic datasets, the Heart dataset has the lowest accuracy, which can likely be attributed %predominantly 
to its relatively small sample size. Other than the Heart dataset, the accuracy of the models trained on the synthetic datasets, for $\epsilon=2$, drops by less than 1\% compared to the real data, and the excess empirical risk is less than 0.02.

\ifsubmissionshort{

\begin{figure} [tb]
\vskip 0.1in
\centering
\begin{minipage}{.3\textwidth}
\begin{subfigure}[Adult Dataset]
{\includegraphics[scale=0.2]{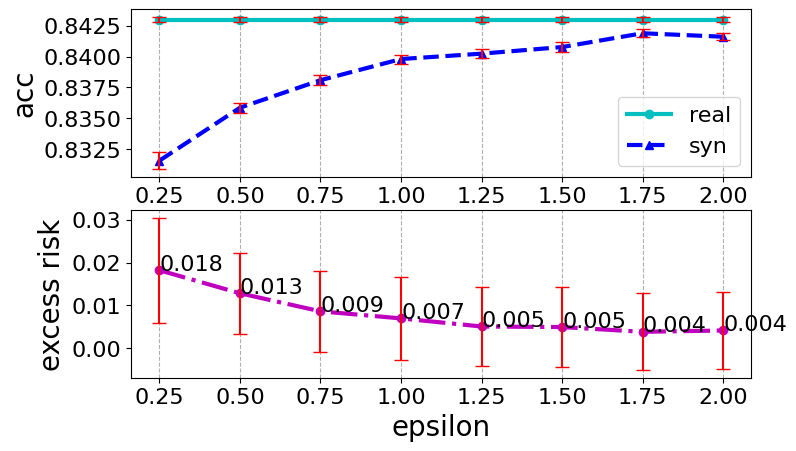}}
\label{fig:adult}
\end{subfigure}
\end{minipage}
\hfill
\begin{minipage}{.3\textwidth}
\begin{subfigure}[Compas Dataset]
{\includegraphics[scale=0.2]{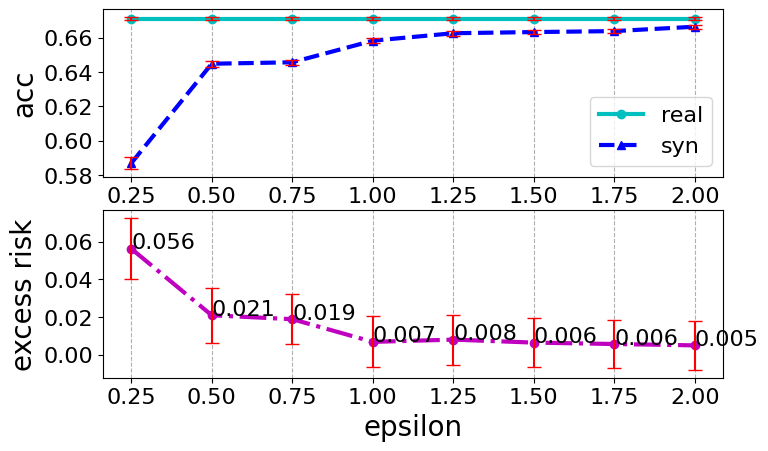}}
\label{fig:compas}
\end{subfigure}
\end{minipage}
\hfill
\begin{minipage}{.3\textwidth}
\begin{subfigure}[Law Dataset]
{\includegraphics[scale=0.2]{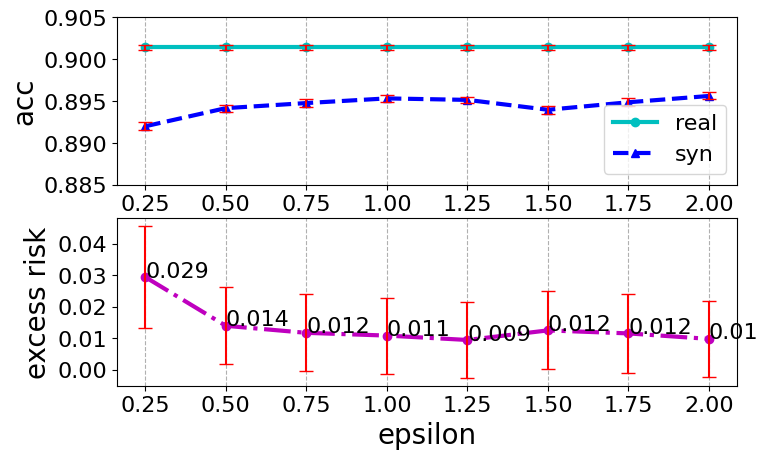}}
\label{fig:law}
\end{subfigure}
\end{minipage}

\begin{minipage}{.3\textwidth}
\begin{subfigure}[Heart Dataset]
{\includegraphics[scale=0.2]{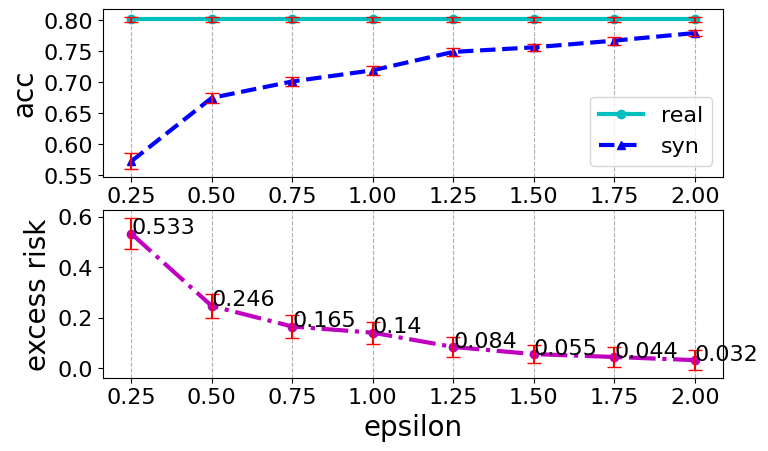}}
\label{fig:heart}
\end{subfigure}
\end{minipage}
\hfill
\begin{minipage}{.3\textwidth}
\begin{subfigure}[Dutch Dataset]
{\includegraphics[scale=0.2]{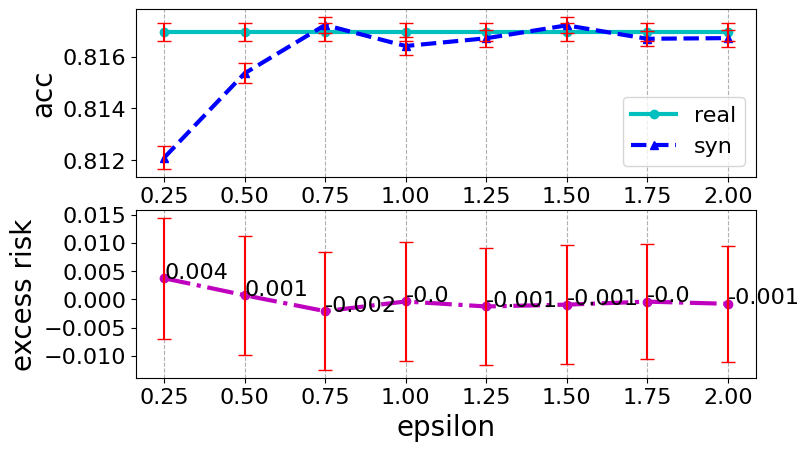}}
\label{fig:dutch}
\end{subfigure}
\end{minipage}
\hfill
\begin{minipage}{.3\textwidth}
\begin{subfigure}[Churn Dataset]
{\includegraphics[scale=0.20]{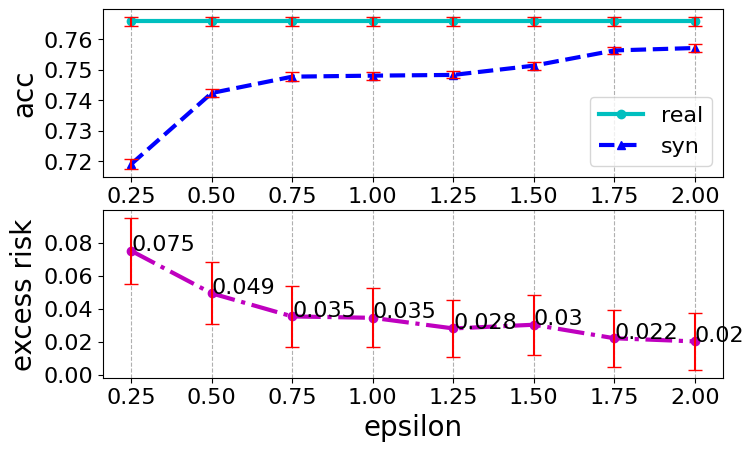}}
\label{fig:churn}
\end{subfigure}
\end{minipage}
% \vskip 0.1in
\caption{We generated synthetic data for the six(6) datasets with $\epsilon \in(\frac{1}{4}, \frac{2}{4}, \frac{3}{4}, 
1, \frac{5}{4}, 
\frac{6}{4}, \frac{7}{4}, 2)$. We produce 10 randomized sets of synthetic data for each $\epsilon$. We assess performance by training the machine learning model 10 times with randomly split datasets to 80\% training, 20\% testing. Note that some degree of minor unpredictability is inevitable due to the limited number of trials, and this causes the slight graph oscillation.}
\label{fig:performance_in_diff_eps}
\vskip -0.1in
\end{figure}

}
\fi

\section{Conclusions and Future Work}
In our study, we give both upper and lower bounds for the excess empirical risk (measured w.r.t.~the real dataset) of training linear models on marginal-preserving synthetic data. Also, we show that for specific ranges of parameter choices, there exists a data distribution such that our upper and lower bounds are nearly tight (both are $1/\mathsf{polylog}(n)$).
Moreover, we give an end-to-end privacy and excess empirical risk analysis for a synthetic data generation mechanism that preserves all $d$-th order marginals. Finally, we supplement our theoretic results with extensive experiments using the AIM mechanism~\citep{aim} to heuristically generate marginal-preserving synthetic datasets for multiple real datasets. Our experiments show that the resulting models, with $\epsilon=2$,  reduce the accuracy by at most $2.2\%$, compared to that of the (non-private) real models.

Moving forward, we believe the following directions are interesting to consider:
\ifsubmissionshort{ 
(1). Given that our experiments on real-world datasets perform significantly better than the lower bound for the worst-case data distribution, it is  interesting to explore assumptions on the data distribution that are consistent with the real-world datasets, and which may allow bypassing the lower bound. 
    (2) It will be interesting to extend our techniques to non-linear models, such as decision trees, SVM, KNN, and neural networks, etc.
    (3) Finally, it will be interesting to broaden our approach to handle data with continuous attributes, or with discrete attributes but very large cardinality. In both cases, the marginals are harder/costlier (in terms of privacy) to preserve, and it may be necessary to develop novel proof techniques.}
\else{
\begin{enumerate}
    \item Given that our experiments on real-world datasets perform much better than the lower bound over worst-case data distribution, it is thus interesting to figure out the right assumption on data distribution capturing the real-world datasets, and develop a more tailored lower bound accordingly.
    \item Additionally, it will be interesting to come up with an efficient differentially private (provably) $d$-th order marginal-preserving synthetic data generation algorithm under such assumptions. 
    \item It will be interesting to see how we can make use of our proving techniques and extend our results to other non-linear models, such as decision trees, SVM, KNN, and neural networks, etc.
    \item Finally, it will be interesting to study our problems on data with continuous attributes, or simply with discrete attributes but very large cardinality. In both cases, the marginals are harder/costlier (in terms of privacy) to preserve, and it may be necessary to develop novel proving techniques.
\end{enumerate}}\fi

\section{Ethical Aspects and Broader Impact}
This paper presents work whose goal is to advance the protection of an individual's privacy in Machine Learning (ML) applications.
Ensuring privacy is a societal concern, and is especially crucial in the ML setting where large amounts of potentially sensitive data are required for training. Furthermore, we believe that the particular methodology put forth in this work--in which differentially private (DP) synthetic data is generated for training---allows for equitable access to training data, in comparison to standard Training-DPML techniques.  Specifically, the DP synthetic data can be released publicly once generated. Further, any out-of-the box optimization algorithm can be run on the data, in contrast to Training-DP algorithms, which require specialized knowledge to properly set the parameters and to run the modified algorithms. Finally, our experiments were performed solely on publicly available data, and we anticipate no potential misuse of the outcomes derived from our research.

\begin{ack}
The authors thank the anonymous reviewers for their insightful comments and valuable feedback on this work.  Dana Dachman-Soled is supported in part by NSF grants CNS-2154705, CNS-1933033, and IIS2147276 is jointly provided with Min Wu.
We gratefully acknowledge support from the JP Morgan Chase Faculty Research Award.

\paragraph{Disclaimer} This paper was prepared for informational purposes in part by the Artificial Intelligence Research group of JPMorgan Chase \& Co and its affiliates (“J.P. Morgan”) and is not a product of the Research Department of J.P. Morgan.  J.P. Morgan makes no representation and warranty whatsoever and disclaims all liability, for the completeness, accuracy or reliability of the information contained herein.  This document is not intended as investment research or investment advice, or a recommendation, offer or solicitation for the purchase or sale of any security, financial instrument, financial product or service, or to be used in any way for evaluating the merits of participating in any transaction, and shall not constitute a solicitation under any jurisdiction or to any person, if such solicitation under such jurisdiction or to such person would be unlawful. 

© 2024 JPMorgan Chase \& Co. All rights reserved.
\end{ack}

\bibliographystyle{plainnat}

\bibliography{refs}

\begin{thebibliography}{51}
\providecommand{\natexlab}[1]{#1}
\providecommand{\url}[1]{\texttt{#1}}
\expandafter\ifx\csname urlstyle\endcsname\relax
  \providecommand{\doi}[1]{doi: #1}\else
  \providecommand{\doi}{doi: \begingroup \urlstyle{rm}\Url}\fi

\bibitem[chu(2020)]{churn}
{Iranian Churn}.
\newblock UCI Machine Learning Repository, 2020.
\newblock {DOI}: https://doi.org/10.24432/C5JW3Z.

\bibitem[Abadi et~al.(2016)Abadi, Chu, Goodfellow, McMahan, Mironov, Talwar, and Zhang]{Abadi_2016}
Martin Abadi, Andy Chu, Ian Goodfellow, H.~Brendan McMahan, Ilya Mironov, Kunal Talwar, and Li~Zhang.
\newblock Deep learning with differential privacy.
\newblock In \emph{Proceedings of the 2016 {ACM} {SIGSAC} Conference on Computer and Communications Security}. {ACM}, oct 2016.
\newblock \doi{10.1145/2976749.2978318}.
\newblock URL \url{https://doi.org/10.1145%2F2976749.2978318}.

\bibitem[Adnan et~al.(2022)Adnan, Kalra, Cresswell, et~al.]{DP-SGD_FL_medical}
M.~Adnan, S.~Kalra, J.C. Cresswell, et~al.
\newblock Federated learning and differential privacy for medical image analysis.
\newblock \emph{Scientific Reports}, 12:\penalty0 1953, 2022.
\newblock \doi{10.1038/s41598-022-05539-7}.

\bibitem[Angwin et~al.(2016)Angwin, Larson, Mattu, and Kirchner]{compas}
J.~Angwin, J.~Larson, S.~Mattu, and L.~Kirchner.
\newblock How we analyzed the compas recidivism algorithm.
\newblock ProPublica, 2016.
\newblock URL \url{https://www.propublica.org/article/how-we-analyzed-the-compas-recidivism-algorithm}.

\bibitem[Avella-Medina et~al.(2023)Avella-Medina, Bradshaw, and Loh]{2nd_order_1}
Marco Avella-Medina, Casey Bradshaw, and Po-Ling Loh.
\newblock Differentially private inference via noisy optimization.
\newblock \emph{The Annals of Statistics}, 51\penalty0 (5):\penalty0 2067--2092, 2023.

\bibitem[Bassily et~al.(2014{\natexlab{a}})Bassily, Smith, and Thakurta]{6979031}
Raef Bassily, Adam Smith, and Abhradeep Thakurta.
\newblock Private empirical risk minimization: Efficient algorithms and tight error bounds.
\newblock In \emph{2014 IEEE 55th Annual Symposium on Foundations of Computer Science}, pages 464--473, 2014{\natexlab{a}}.
\newblock \doi{10.1109/FOCS.2014.56}.

\bibitem[Bassily et~al.(2014{\natexlab{b}})Bassily, Smith, and Thakurta]{private_SGD_2}
Raef Bassily, Adam Smith, and Abhradeep Thakurta.
\newblock Private empirical risk minimization: Efficient algorithms and tight error bounds.
\newblock In \emph{2014 IEEE 55th annual symposium on foundations of computer science}, pages 464--473. IEEE, 2014{\natexlab{b}}.

\bibitem[Becker and Kohavi(1996)]{adult}
Barry Becker and Ronny Kohavi.
\newblock {Adult}.
\newblock UCI Machine Learning Repository, 1996.
\newblock {DOI}: https://doi.org/10.24432/C5XW20.

\bibitem[Bernstein(1912)]{bernstein1912proof}
S~Bernstein.
\newblock Proof of the theorem of weierstrass based on the calculus of probabilities.
\newblock \emph{Communications of the Kharkov Mathematical Society}, 13:\penalty0 1--2, 1912.

\bibitem[Cai et~al.(2021)Cai, Lei, Wei, and Xiao]{PrivMRF}
Kuntai Cai, Xiaoyu Lei, Jianxin Wei, and Xiaokui Xiao.
\newblock Data synthesis via differentially private markov random fields.
\newblock \emph{Proc. VLDB Endow.}, 14\penalty0 (11):\penalty0 2190–2202, jul 2021.
\newblock ISSN 2150-8097.
\newblock \doi{10.14778/3476249.3476272}.
\newblock URL \url{https://doi.org/10.14778/3476249.3476272}.

\bibitem[Canonne et~al.(2020)Canonne, Kamath, and Steinke]{zCDP_to_DP}
Cl{\'e}ment~L Canonne, Gautam Kamath, and Thomas Steinke.
\newblock The discrete gaussian for differential privacy.
\newblock \emph{Advances in Neural Information Processing Systems}, 33:\penalty0 15676--15688, 2020.

\bibitem[Centraal Bureau voor~de Statistiek (CBS) (Statistics~Netherlands)(2016-04-25)]{dutch}
Minnesota Population~Center Centraal Bureau voor~de Statistiek (CBS) (Statistics~Netherlands).
\newblock The dutch virtual census of 2001 - ipums subset.
\newblock Integrated Public Use Microdata Series (IPUMS) [dataset]. Minneapolis: University of Minnesota, 2015, 2016-04-25.
\newblock URL \url{https://microdata.worldbank.org/index.php/catalog/2102}.
\newblock {DOI}: 10.18128/D020.V6.4.

\bibitem[Chaudhuri et~al.(2011)Chaudhuri, Monteleoni, and Sarwate]{chaudhuri2011differentially}
Kamalika Chaudhuri, Claire Monteleoni, and Anand~D Sarwate.
\newblock Differentially private empirical risk minimization.
\newblock \emph{Journal of Machine Learning Research}, 12\penalty0 (3), 2011.

\bibitem[Dagan and Feldman(2020)]{DF20}
Yuval Dagan and Vitaly Feldman.
\newblock Interaction is necessary for distributed learning with privacy or communication constraints.
\newblock In \emph{Proceedings of the 52nd Annual ACM SIGACT Symposium on Theory of Computing}, STOC 2020, page 450–462, New York, NY, USA, 2020. Association for Computing Machinery.
\newblock ISBN 9781450369794.
\newblock \doi{10.1145/3357713.3384315}.
\newblock URL \url{https://doi.org/10.1145/3357713.3384315}.

\bibitem[Davis(1975)]{davis1975interpolation}
Philip~J Davis.
\newblock \emph{Interpolation and approximation}.
\newblock Courier Corporation, 1975.

\bibitem[De et~al.(2022)De, Berrada, Hayes, Smith, and Balle]{DP-SGD_image}
Soham De, Leonard Berrada, Jamie Hayes, Samuel~L Smith, and Borja Balle.
\newblock Unlocking high-accuracy differentially private image classification through scale.
\newblock \emph{arXiv preprint arXiv:2204.13650}, 2022.

\bibitem[Dupuy et~al.(2022)Dupuy, Arava, Gupta, and Rumshisky]{DP-SGD_NLU}
Christophe Dupuy, Radhika Arava, Rahul Gupta, and Anna Rumshisky.
\newblock An efficient dp-sgd mechanism for large scale nlu models.
\newblock In \emph{ICASSP 2022 - 2022 IEEE International Conference on Acoustics, Speech and Signal Processing (ICASSP)}, pages 4118--4122, 2022.
\newblock \doi{10.1109/ICASSP43922.2022.9746975}.

\bibitem[Dwork and Roth(2014)]{dwork2014algorithmic}
Cynthia Dwork and Aaron Roth.
\newblock The algorithmic foundations of differential privacy.
\newblock \emph{Foundations and Trends in Theoretical Computer Science}, 9\penalty0 (3-4):\penalty0 211--407, 2014.

\bibitem[Dwork et~al.(2006)Dwork, McSherry, Nissim, and Smith]{dwork2006}
Cynthia Dwork, Frank McSherry, Kobbi Nissim, and Adam Smith.
\newblock Calibrating noise to sensitivity in private data analysis.
\newblock In Shai Halevi and Tal Rabin, editors, \emph{Theory of Cryptography}, pages 265--284, Berlin, Heidelberg, 2006. Springer Berlin Heidelberg.
\newblock ISBN 978-3-540-32732-5.

\bibitem[Esteban et~al.(2017)Esteban, Hyland, and R{\"a}tsch]{TSRS}
Crist{\'o}bal Esteban, Stephanie~L. Hyland, and Gunnar R{\"a}tsch.
\newblock Real-valued (medical) time series generation with recurrent conditional gans.
\newblock \emph{ArXiv}, abs/1706.02633, 2017.
\newblock URL \url{https://api.semanticscholar.org/CorpusID:29681354}.

\bibitem[Fredrikson et~al.(2014)Fredrikson, Lantz, Jha, Lin, Page, and Ristenpart]{fredrikson2014privacy}
Matthew Fredrikson, Eric Lantz, Somesh Jha, Simon Lin, David Page, and Thomas Ristenpart.
\newblock Privacy in pharmacogenetics: An end-to-end case study of personalized warfarin dosing.
\newblock In \emph{USENIX Security Symposium}, pages 17--32, 2014.

\bibitem[Ganesh et~al.(2024)Ganesh, Haghifam, Steinke, and Guha~Thakurta]{2nd_order_2}
Arun Ganesh, Mahdi Haghifam, Thomas Steinke, and Abhradeep Guha~Thakurta.
\newblock Faster differentially private convex optimization via second-order methods.
\newblock \emph{Advances in Neural Information Processing Systems}, 36, 2024.

\bibitem[Gersho and Gray(1992)]{quantization_book}
Allen Gersho and Robert~M. Gray.
\newblock \emph{Vector Quantization and Signal Compression}.
\newblock The Springer International Series in Engineering and Computer Science. Springer, New York, NY, 1 edition, 1992.
\newblock ISBN 978-0-7923-9181-4.
\newblock \doi{10.1007/978-1-4615-3626-0}.
\newblock Published: 30 November 1991.

\bibitem[Guan(2009)]{guan2009iterated}
Zhong Guan.
\newblock Iterated bernstein polynomial approximations.
\newblock \emph{arXiv}, preprint arXiv:0909.0684, 2009.

\bibitem[Iyengar et~al.(2019)Iyengar, Near, Song, Thakkar, Thakurta, and Wang]{iyengar2019towards}
Roger Iyengar, Joseph~P Near, Dawn Song, Om~Thakkar, Abhradeep Thakurta, and Lun Wang.
\newblock Towards practical differentially private convex optimization.
\newblock In \emph{2019 IEEE Symposium on Security and Privacy (SP)}, pages 299--316. IEEE, 2019.

\bibitem[Janosi et~al.(1988)Janosi, Steinbrunn, Pfisterer, and Detrano]{heart}
Andras Janosi, William Steinbrunn, Matthias Pfisterer, and Robert Detrano.
\newblock {Heart Disease}.
\newblock UCI Machine Learning Repository, 1988.
\newblock {DOI}: https://doi.org/10.24432/C52P4X.

\bibitem[Jayaraman et~al.(2018)Jayaraman, Wang, Evans, and Gu]{DP-GD_0}
Bargav Jayaraman, Lingxiao Wang, David Evans, and Quanquan Gu.
\newblock Distributed learning without distress: Privacy-preserving empirical risk minimization.
\newblock In S.~Bengio, H.~Wallach, H.~Larochelle, K.~Grauman, N.~Cesa-Bianchi, and R.~Garnett, editors, \emph{Advances in Neural Information Processing Systems}, volume~31. Curran Associates, Inc., 2018.
\newblock URL \url{https://proceedings.neurips.cc/paper_files/paper/2018/file/7221e5c8ec6b08ef6d3f9ff3ce6eb1d1-Paper.pdf}.

\bibitem[Li et~al.(2023)Li, Wang, and Cheng]{li2023statistical}
Ximing Li, Chendi Wang, and Guang Cheng.
\newblock Statistical theory of differentially private marginal-based data synthesis algorithms.
\newblock \emph{arXiv preprint arXiv:2301.08844}, 2023.

\bibitem[Liu et~al.(2021)Liu, Vietri, and Wu]{PEP}
Terrance Liu, Giuseppe Vietri, and Zhiwei~Steven Wu.
\newblock Iterative methods for private synthetic data: Unifying framework and new methods.
\newblock \emph{arXiv}, 2106.07153, 2021.

\bibitem[Lyu et~al.(2020)Lyu, Li, Nandakumar, Yu, and Ma]{DP-SGD_fairness_decentralized}
Lingjuan Lyu, Yitong Li, Karthik Nandakumar, Jiangshan Yu, and Xingjun Ma.
\newblock How to democratise and protect ai: Fair and differentially private decentralised deep learning.
\newblock \emph{IEEE Transactions on Dependable and Secure Computing}, page 1–1, 2020.
\newblock ISSN 2160-9209.
\newblock \doi{10.1109/tdsc.2020.3006287}.
\newblock URL \url{http://dx.doi.org/10.1109/TDSC.2020.3006287}.

\bibitem[Malekzadeh et~al.(2021)Malekzadeh, Hasircioglu, Mital, Katarya, Ozfatura, and G{\"u}nd{\"u}z]{DP-SGD_medical}
Mohammad Malekzadeh, Burak Hasircioglu, Nitish Mital, Kunal Katarya, Mehmet~Emre Ozfatura, and Deniz G{\"u}nd{\"u}z.
\newblock Dopamine: Differentially private federated learning on medical data.
\newblock \emph{arXiv preprint arXiv:2101.11693}, 2021.

\bibitem[McKenna et~al.(2019)McKenna, Sheldon, and Miklau]{Graphical-model}
Ryan McKenna, Daniel Sheldon, and Gerome Miklau.
\newblock Graphical-model based estimation and inference for differential privacy.
\newblock \emph{CoRR}, abs/1901.09136, 2019.
\newblock URL \url{http://arxiv.org/abs/1901.09136}.

\bibitem[McKenna et~al.(2021)McKenna, Miklau, and Sheldon]{DBLP:journals/corr/abs-2108-04978}
Ryan McKenna, Gerome Miklau, and Daniel Sheldon.
\newblock Winning the {NIST} contest: {A} scalable and general approach to differentially private synthetic data.
\newblock \emph{CoRR}, abs/2108.04978, 2021.
\newblock URL \url{https://arxiv.org/abs/2108.04978}.

\bibitem[McKenna et~al.(2022)McKenna, Mullins, Sheldon, and Miklau]{aim}
Ryan McKenna, Brett Mullins, Daniel Sheldon, and Gerome Miklau.
\newblock Aim: An adaptive and iterative mechanism for differentially private synthetic data.
\newblock \emph{arXiv preprint arXiv:2201.12677}, 2022.

\bibitem[McMahan et~al.(2017)McMahan, Ramage, Talwar, and Zhang]{DP-SGD_FL_language}
H~Brendan McMahan, Daniel Ramage, Kunal Talwar, and Li~Zhang.
\newblock Learning differentially private recurrent language models.
\newblock \emph{arXiv preprint arXiv:1710.06963}, 2017.

\bibitem[McSherry and Talwar(2007)]{MT07}
Frank McSherry and Kunal Talwar.
\newblock Mechanism design via differential privacy.
\newblock In \emph{Proceedings - Annual IEEE Symposium on Foundations of Computer Science, FOCS}, pages 94--103, 11 2007.
\newblock ISBN 978-0-7695-3010-9.
\newblock \doi{10.1109/FOCS.2007.66}.

\bibitem[Papernot and Steinke(2021)]{tunning_hyperparameter}
Nicolas Papernot and Thomas Steinke.
\newblock Hyperparameter tuning with renyi differential privacy.
\newblock \emph{arXiv preprint arXiv:2110.03620}, 2021.

\bibitem[Papernot et~al.(2016)Papernot, Abadi, Erlingsson, Goodfellow, and Talwar]{papernot2017semisupervised}
Nicolas Papernot, Mart{\'\i}n Abadi, Ulfar Erlingsson, Ian Goodfellow, and Kunal Talwar.
\newblock Semi-supervised knowledge transfer for deep learning from private training data.
\newblock \emph{arXiv preprint arXiv:1610.05755}, 2016.

\bibitem[Papernot et~al.(2018)Papernot, Song, Mironov, Raghunathan, Talwar, and Erlingsson]{papernot2018scalable}
Nicolas Papernot, Shuang Song, Ilya Mironov, Ananth Raghunathan, Kunal Talwar, and {\'U}lfar Erlingsson.
\newblock Scalable private learning with pate.
\newblock \emph{arXiv preprint arXiv:1802.08908}, 2018.

\bibitem[Pedregosa et~al.(2011)Pedregosa, Varoquaux, Gramfort, Michel, Thirion, Grisel, Blondel, Prettenhofer, Weiss, Dubourg, Vanderplas, Passos, Cournapeau, Brucher, Perrot, and Duchesnay]{scikit-learn}
F.~Pedregosa, G.~Varoquaux, A.~Gramfort, V.~Michel, B.~Thirion, O.~Grisel, M.~Blondel, P.~Prettenhofer, R.~Weiss, V.~Dubourg, J.~Vanderplas, A.~Passos, D.~Cournapeau, M.~Brucher, M.~Perrot, and E.~Duchesnay.
\newblock Scikit-learn: Machine learning in {P}ython.
\newblock \emph{Journal of Machine Learning Research}, 12:\penalty0 2825--2830, 2011.

\bibitem[Popoviciu(1935)]{bernstein_r_0}
T.~Popoviciu.
\newblock Sur l’approximation des fonctions convexes d’ordre sup\'erieur.
\newblock \emph{Mathematica (Cluj)}, 10:\penalty0 49–54, 1935.

\bibitem[Roulier(1970)]{ROULIER1970117_bernstein}
John~A Roulier.
\newblock Permissible bounds on the coefficients of approximating polynomials.
\newblock \emph{Journal of Approximation Theory}, 3\penalty0 (2):\penalty0 117--122, 1970.
\newblock ISSN 0021-9045.
\newblock \doi{https://doi.org/10.1016/0021-9045(70)90018-3}.
\newblock URL \url{https://www.sciencedirect.com/science/article/pii/0021904570900183}.

\bibitem[Shokri et~al.(2017)Shokri, Stronati, Song, and Shmatikov]{shokri2017membership}
Reza Shokri, Marco Stronati, Congzheng Song, and Vitaly Shmatikov.
\newblock Membership inference attacks against machine learning models.
\newblock In \emph{2017 IEEE symposium on security and privacy (SP)}, pages 3--18. IEEE, 2017.

\bibitem[Song et~al.(2013)Song, Chaudhuri, and Sarwate]{private_SGD_1}
Shuang Song, Kamalika Chaudhuri, and Anand~D. Sarwate.
\newblock Stochastic gradient descent with differentially private updates.
\newblock In \emph{2013 IEEE Global Conference on Signal and Information Processing}, pages 245--248, 2013.
\newblock \doi{10.1109/GlobalSIP.2013.6736861}.

\bibitem[Telyakovskii(2009)]{bernstein_rth_order}
S.A. Telyakovskii.
\newblock On the rate of approximation of functions by the bernstein polynomials.
\newblock \emph{Proc. Steklov Inst. Math.}, 264\penalty0 (Suppl 1):\penalty0 177–184, 2009.
\newblock \doi{10.1134/S0081543809050150}.
\newblock URL \url{https://doi.org/10.1134/S0081543809050150}.

\bibitem[Wang et~al.(2021)Wang, Fu, Li, Khisti, Zemel, and Makhzani]{wang2022variational}
Kuan-Chieh Wang, Yan Fu, Ke~Li, Ashish Khisti, Richard Zemel, and Alireza Makhzani.
\newblock Variational model inversion attacks.
\newblock \emph{Advances in Neural Information Processing Systems}, 34:\penalty0 9706--9719, 2021.

\bibitem[Wightman(1998)]{law}
Linda Wightman.
\newblock Lsac national longitudinal bar passage study.
\newblock LSAC Research Report Series. ERIC, 1998.

\bibitem[Yu et~al.(2021)Yu, Zhang, Chen, Yin, and Liu]{yu2021gradient}
Da~Yu, Huishuai Zhang, Wei Chen, Jian Yin, and Tie-Yan Liu.
\newblock Gradient perturbation is underrated for differentially private convex optimization.
\newblock In \emph{Proceedings of the Twenty-Ninth International Conference on International Joint Conferences on Artificial Intelligence}, pages 3117--3123, 2021.

\bibitem[Zhang et~al.(2017{\natexlab{a}})Zhang, Zheng, Mou, and Wang]{zhang2017efficient}
Jiaqi Zhang, Kai Zheng, Wenlong Mou, and Liwei Wang.
\newblock Efficient private erm for smooth objectives.
\newblock In \emph{Proceedings of the 26th International Joint Conference on Artificial Intelligence}, pages 3922--3928, 2017{\natexlab{a}}.

\bibitem[Zhang et~al.(2017{\natexlab{b}})Zhang, Cormode, Procopiuc, Srivastava, and Xiao]{PrivBayes}
Jun Zhang, Graham Cormode, Cecilia~M. Procopiuc, Divesh Srivastava, and Xiaokui Xiao.
\newblock Privbayes: Private data release via bayesian networks.
\newblock \emph{ACM Trans. Database Syst.}, 42\penalty0 (4), oct 2017{\natexlab{b}}.
\newblock ISSN 0362-5915.
\newblock \doi{10.1145/3134428}.
\newblock URL \url{https://doi.org/10.1145/3134428}.

\bibitem[Zhang et~al.(2021)Zhang, Wang, Li, Honorio, Backes, He, Chen, and Zhang]{PrivSyn}
Zhikun Zhang, Tianhao Wang, Ninghui Li, Jean Honorio, Michael Backes, Shibo He, Jiming Chen, and Yang Zhang.
\newblock {PrivSyn}: Differentially private data synthesis.
\newblock In \emph{30th USENIX Security Symposium (USENIX Security 21)}, pages 929--946. USENIX Association, August 2021.
\newblock ISBN 978-1-939133-24-3.
\newblock URL \url{https://www.usenix.org/conference/usenixsecurity21/presentation/zhang-zhikun}.

\end{thebibliography}

% \appendix{Appendix}\\
\newpage
\appendix

\ifsubmission
\begin{center}
{\huge Supplementary Material}
\end{center}
\fi

\sloppy

\onecolumn

\ifsubmissionshort{
\section{Proofs in Section~\ref{sec:upper_bound}}
\subsection{Proof of Theorem~\ref{thm:ASD_gen}}\label{app:ASD_gen}
\begin{proof}
Our proof relies on the following empirical risk function $L'$ that approximates $L$:
\begin{align*}
    L'(\vec{w},D)=\frac{1}{n} \sum_{(\vec{x},y)\in D} P_{d-1}\varphi(\la \vec{w},\vec{x} \ra y),
\end{align*}
where $P_{d-1}\varphi$ is the degree-$(d-1)$ Bernstein polynomial to approximate $\varphi$ within the interval $[- \tau \sqrt{m}, \tau \sqrt{m}]$ (or $[-1,-1]$ if $\tau \sqrt{m} < 1$).

\begin{lemma} \label{lem:delta_poly_gen}
For any normalized dataset $D$ and any $\vec{w}$ such that $\|\vec{w}\|_2 \leq \tau$,
\[
|L'(\vec{w},D) - L(\vec{w},D)| \in O(K\cdot\tau\sqrt{m/(d-1)}).
\]
\end{lemma}
\begin{proof}

It suffices to show that for any $(\vec{x},y)\in D$, $|P_{d-1}\varphi(\la \vec{w},\vec{x} \ra y) - \varphi(\la \vec{w},\vec{x} \ra y)| \in O(K\cdot\tau \sqrt{m/(d-1)})$.

First,  we have $|\la \vec{w},\vec{x} \ra y|\leq \|\vec{w}\|_2\|\vec{x}\|_2 \leq \tau \sqrt{m}$, where the first inequality follows from Cauchy–Schwarz inequality and $y \in \{-1,1\}$, and the second inequality follows from $\|\vec{w}\|_2 \leq \tau$.

Next, using the approximation error of Bernstein polynomial (Theorem \ref{bernstein_1}, Eq. \ref{bern_err} ), we have the maximum error $|P_{d-1}\varphi(\la \vec{w},\vec{x} \ra y) - \varphi(\la \vec{w},\vec{x} \ra y)| \in O( \omega(\varphi,\frac{\tau \sqrt{m}}{\sqrt{d-1}}))$ for any $\la \vec{w},\vec{x} \ra y \in [-\tau \sqrt{m},\tau \sqrt{m}]$. As $\varphi$ is $K$-Lipschitz, $\omega(\varphi,\frac{\tau \sqrt{m}}{\sqrt{d-1}}) \leq K \cdot \frac{\tau \sqrt{m}}{\sqrt{d-1}}$.
\end{proof}

Next, we bound the empirical risk difference using $L'$ between the real and synthetic datasets on any $\vec{w}$.

\begin{lemma} \label{lem:delta_mar_gen}
% Given any dataset $D := [\vec{x}_i, y_i]_{i \in [n]}$, 
For any  $\vec{w}$ such that $\|\vec{w}\|_2 \leq \tau$, and any datasets $D_r$ and $D_s$ such that for all $q \in Q^m_{\leq d}$, $\|\vec{h}_{q}^{(r)} - \vec{h}_{q}^{(s)}\|_1 \leq \nu$, we have

\ifsubmissionshort{
\begin{align*}
&|L'(\vec{w},D_r) - L'(\vec{w},D_s)| \\
\in 
& O\left(\frac{1}{n} \cdot (K\tau\sqrt{m}+ \varphi(0)) \cdot (3m \cdot \max\{1,\tau\})^{d-1} \nu \right).
\end{align*}
}
\else{
\begin{align*}
|L'(\vec{w},D_r) - L'(\vec{w},D_s)| \in 
O\left(\frac{1}{n} \cdot (K\tau\sqrt{m}+ \varphi(0)) \cdot (3m \cdot \max\{1,\tau\})^{d-1} \nu \right).
\end{align*}}
\fi

\end{lemma}

\begin{proof}
We start by expressing the $L'$ of dataset $D$ on $\vec{w}$ using $D$'s marginals with order no more than $d$:
\begin{align*}
    L'(\vec{w},D)=&\frac{1}{n}\sum_{(\vec{x},y)\in D}P_{d-1}\varphi(\la \vec{w},\vec{x} \ra y)\\
    =&\frac{1}{n}\sum_{(\vec{x},y)\in D}\sum_{k=0}^{d-1} a_{k} (\la \vec{w},\vec{x} \ra y)^k\\
    =&\frac{1}{n}\sum_{(\vec{x},y)\in D}\sum_{k=0}^{d-1} a_{k}  \sum_{\vec{u} \in[m]^{k}} \prod_{u \in \vec{u}}(\vec{w}[u]\cdot \vec{x}[u] \cdot y)\\
    =&\frac{1}{n}\sum_{k=0}^{d-1} a_{k}  \sum_{\vec{u} \in[m]^{k}} \sum_{\vec{t}\in\Omega_{q}, q = \mathcal{S}(\vec{u}) \cup\{m+1\}} \vec{h}_q(\vec{t}) \prod_{u \in \vec{u}}(\vec{w}[u]\cdot \vec{t}[u] \cdot \vec{t}[m+1]),\\
% (\la \vec{w},\vec{x} \ra y)^k.
\end{align*}
where $\mathcal{S}(\vec{u})$ returns a set containing unique elements in $\vec{u}$ and the entries of $\vec{t}$ is indexed by the set $q$.

The above expression allows us to bound the empirical risk difference between real and synthetic datasets using the bounded difference of their marginals. Specifically, 

\begin{align*}
&|L'(\vec{w},D_r) - L'(\vec{w},D_s)|\\
    =&\left|\frac{1}{n}\sum_{k=0}^{d-1} a_{k}  \sum_{\vec{u} \in[m]^{k}} \sum_{\vec{t}\in\Omega_{q}, q = \mathcal{S}(\vec{u}) \cup\{m+1\}} (\vec{h}_q^{r}(\vec{t})-\vec{h}_q^{s}(\vec{t})) \prod_{u \in \vec{u}}(\vec{w}[u]\cdot \vec{t}[u] \cdot \vec{t}[m+1])\right|\\
    \leq&\left|\frac{1}{n}\sum_{k=0}^{d-1} a_{k}\cdot \sum_{\vec{u} \in[m]^{k}}  \|\vec{h}_q^{(r)}(\vec{t})-\vec{h}_q^{(s)}(\vec{t})\|_1\tau^k\right|\\
    \leq&\left|\frac{1}{n}\sum_{k=0}^{d-1} a_{k}\cdot \sum_{\vec{u} \in[m]^{k}} \nu \tau^k\right|\\
    \leq&\left|\frac{1}{n}\sum_{k=0}^{d-1} a_{k}\cdot m^{d-1} \nu \max\{1,\tau\}^{d-1}\right|\\
    \leq&\frac{1}{n}\sum_{k=0}^{d-1} |a_{k}|\cdot m^{d-1} \nu \max\{1,\tau\}^{d-1}\\
    \in& O\left(\frac{1}{n} \cdot (K \cdot \tau\sqrt{m}+\varphi(0))\cdot 3^{d-1} \cdot (m\cdot \max\{1,\tau\})^{d-1} \nu \right),\\
    % \leq&\left|\frac{-1}{n} \sum_{k=0}^{d-1}  a_k \cdot \sum_{\vec{u} \in [m]^k} \nu \cdot \tau^{k/2}\right|\\
    % \in&O\left(\frac{1}{n} \cdot \tau\sqrt{m}\cdot 3^{d} \cdot (m+1)^d \nu \cdot \tau^{d/2}\right).\\
\end{align*}
where the first inequality follows from $\|\vec{w}\|_2\leq \tau$ and $\vec{t}[m+1] \in \{-1,1\}$, and the last expression follows by applying Theorem \ref{bernstein_1}, Eq. \ref{bern_coef} to bound the sum of the absolute values of the polynomial coefficients.

\end{proof}

We are ready to prove the Theorem statement by combining the results in Lemmas \ref{lem:delta_poly_gen} and \ref{lem:delta_mar_gen}. Specifically, We write $ A \stackrel{P}{\approx} B$ to denote the LHS and RHS is bounded by the error due in Lemma \ref{lem:delta_poly_gen} and write $ A\stackrel{M}{\approx} B$ to denote the LHS and RHS is bounded by the error due in Lemma \ref{lem:delta_mar_gen}.

Let $\vec{w}'_r = \mathsf{argmin}_\vec{w} L'(\vec{w},D_r)$ and $\vec{w}'_s = \mathsf{argmin}_\vec{w} L'(\vec{w},D_s)$. (In the case that there is more than one minimums, it suffices to use arbitrary tie-breaking.) Then we have:
\begin{align*}
    &L(\vec{w}_s, D_r) \stackrel{P}{\approx}  L'(\vec{w}_s, D_r)\stackrel{M}{\approx}  L'(\vec{w}_s, D_s) \stackrel{P}{\approx} L(\vec{w}_s, D_s) \leq L(\vec{w}'_s, D_s)\\
    &\quad\quad\quad\stackrel{P}{\approx}L'(\vec{w}'_s, D_s) \leq L'(\vec{w}'_r, D_s)\stackrel{M}{\approx}L'(\vec{w}'_r, D_r) \leq  L'(\vec{w}_r, D_r) \stackrel{P}{\approx}L(\vec{w}_r, D_r),
\end{align*}
where the inequalities follows the optimality of $\vec{w}'_s,\vec{w}_s,\vec{w}'_r,\vec{w}_r$. This suggests $L(\vec{w}_s, D_r) - L(\vec{w}_r, D_r) \in O\left(\frac{1}{n} \cdot (K\tau\sqrt{m}+ \varphi(0)) \cdot (3m \cdot \max\{1,\tau\})^{d-1} \nu \right)$. Similarly, we have:
\begin{align*}
    &L(\vec{w}_s, D_r) \stackrel{P}{\approx}  L'(\vec{w}_s, D_r)\stackrel{M}{\approx}  L'(\vec{w}_s, D_s) \geq L'(\vec{w}'_s, D_s) \\
    &\quad\quad\quad\stackrel{M}{\approx}  L'(\vec{w}'_s, D_r)\geq L'(\vec{w}'_r, D_r) \stackrel{P}{\approx}  L(\vec{w}'_r, D_r) \geq L(\vec{w}_r, D_r), 
\end{align*}
which suggests $L(\vec{w}_r, D_r) - L(\vec{w}_s, D_r) \in O\left(\frac{1}{n} \cdot (K\tau\sqrt{m}+ \varphi(0)) \cdot (3m \cdot \max\{1,\tau\})^{d-1} \nu \right)$.
This concludes our proof of the Theorem.

\end{proof}

\subsection{Proof of Theorem~\ref{thm:ASD_log}}\label{app:ASD_log}
\begin{proof}
The majority of the proof is the same as that of Theorem \ref{thm:ASD_gen}. By using the additional property the first derivative of $\hat{\varphi}$ is continuous and its first derivative is $1/4$-Lipschitz, we can apply Theorem \ref{bernstein_2} to give a tighter bound of polynomial approximation error. Note that while Theorem \ref{bernstein_2} only considers functions defined over $[0,1]$, we can shrink any function defined over $[a,b]$ into this range. In the case of $\hat{\varphi}$, this results in the Lipschitz constant of its first derivative multiplied by $(b-a) = \tau \sqrt{m}$. Therefore, we have
\begin{align*}
    |\hat{L}'(\vec{w},D) - \hat{L}(\vec{w},D)| \in O(K\tau\sqrt{m}/(d-1)).
\end{align*}
Finally, by plugging in $K=1$ and $\|\hat{\varphi}\| \leq \ln(2)+ \tau\sqrt{m}$ for logistic loss yields the result.

\end{proof}

\subsection{Proof of Lemma~\ref{lem:syn_dp}}\label{app:syn_dp}
\begin{proof}
Recall each marginal is a vector of counts, where altering a single data point can, at most, result in a difference of 1 in two counts. Therefore, the $\ell_2$ sensitivity of the concatenated marginals is $\sqrt{2|Q^m_{\leq d}|} \leq \sqrt{2 m^d}$. By Theorem \ref{thm:gauss}, the noisy marginals satisfies $(\epsilon,\delta)$-DP. As the synthetic data is exclusively generated using these noisy marginals, it also satisfies $(\epsilon, \delta)$-DP through post-processing (Theorem \ref{thm:post_proc}).

\end{proof}

\subsection{Proof of Lemma~\ref{lem:syn_bound}}\label{app:syn_bound}
\begin{proof}
Note that for any $q \in Q^m_{\leq d}$ and any $\vec{t} \in \Omega_q$, $\vec{\hat{h}}_q(\vec{t}) - \vec{h}_q^{(r)}(\vec{t}) \sim \mathcal{N}(0,\sigma)$. Using Chernoff bound, $|\vec{\hat{h}}_q(\vec{t}) - \vec{h}_q^{(r)}(\vec{t})| \leq k\sigma$ with $1 - 2 e^{-k^2/2}$ probability. Using Union bound, $\|\vec{\hat{h}}_q - \vec{h}_q^{(r)}\|_1 = \sum_{t \in \Omega_q}|\vec{\hat{h}}_q(\vec{t}) - \vec{h}_q^{(r)}(\vec{t})|\leq l^dk\sigma$ with $1 - 2 l^d e^{-k^2/2}$ probability. 

By definition of $D_s$, we have $\max_{q \in Q^m_{\leq d}} \|\vec{\hat{h}}_q - M_q(D_s)\|_1 \leq \max_{q \in Q^m_{\leq d}} \|\vec{\hat{h}}_q - M_q(D_r)\|_1$. Therefore, using triangle inequality, we have $\max_{q \in Q^m_{\leq d}} \|\vec{h}_q^{(r)} - M_q(D_s)\|_1 \leq 2 l^dk\sigma$ with $1 - 2 l^d e^{-k^2/2}$ probability. Finally, using union bound, we have the above inequality holds for all $q\in Q^m_{\leq d}$ with $1-2 (ml)^d e^{-k^2/2}$.

By setting $k=\sqrt{2(\ln(2)(1+\lambda)+d\ln(ml))}$
% $k= \sqrt{\frac{2(\lambda +1+d \lg(ml))}{\lg e}}$ 
concludes our proof.
\end{proof}

\section{Proof of Theorem~\ref{th:lower_bound}}
\label{app:theorem 4.1 proof}
\begin{proof}
We begin by setting parameters
$r, n, m, d, \gamma, \tau$ as follows:
\begin{definition}[Parameter Settings] \label{def:params}
We set parameters as follows:
\begin{itemize}
\item Set $r = 5/6$.
\item Set $m > 2e$.
\item Set $\gamma = (m/2)^{\frac{-5}{10-2r}}$.
\item Set $d = \frac{c' \cdot \gamma^{-2r/5}}{-\ln(\gamma)}$, for $c'=\min\{\frac{1}{5},\frac{c}{8}\}$, where $c$ is a constant depending only on $r$ (See Theorem \ref{th:stat_query_lb}).

\item Set $n = \exp(\gamma^{-2r/5})$.
\item Set $\tau = \frac{1}{\sqrt{m}}$.

\end{itemize}
\end{definition}

We next define the loss function which will be used for both the upper bound and the lower bound.

Consider the following convex loss function
$\varphi_\gamma : [-1, 1] \to \mathbb{R}$
defined in~\cite{DF20}:

\begin{equation}
    \varphi_\gamma(t) = \frac{(1-t)^2}{8} +
    \begin{cases*}
      1-2t/\gamma & $-1 \leq t \leq 0$ \\
      (t-\gamma)^2/\gamma^2 & $0 \leq t \leq \gamma$\\
      0 & $\gamma \leq t \leq 1$.
    \end{cases*}
  \end{equation}

The loss function $\varphi$ from Theorem~\ref{th:lower_bound}
is set to be $\varphi(t) := \gamma \cdot \varphi_\gamma(t)$.

\begin{claim} \label{claim:approx_f}
%Let $f(t) := \gamma \cdot \varphi_\gamma(t)$.
Let $P_d\varphi(x)$ be the Bernstein polynomial of order $d$ of $\varphi$
on $[-1, 1]$. Then
\[
||P_d\varphi - \varphi|| \leq \frac{5}{\sqrt{d}}.
\]
\end{claim}

The claim follows from Theorem~\ref{bernstein_1}, Eq.~\ref{bern_err}
and the fact that $\varphi$ is $2$-Lipschitz. Note that $\gamma <1$.\\
% that the modulus of continuity
% $\omega(f, \frac{2}{\sqrt{d}}) = \frac{4}{\sqrt{d}}$.

We now turn to the upper bound (the first item in Theorem~\ref{th:lower_bound}).
For the upper bound, the algorithm $\mathsf{Syn}(n, \{n\cdot\vec{u}_q\}_{q \in Q^m_{\leq d}})$ will return the database $D_s$ of size $n$ in the support of $\mathcal{D}^n_m$ that minimizes $\max_{q \in Q^m_{\leq d}}\|\vec{h}_q^{(s)} - n\cdot\vec{u}_q\|_1$,
where $\{\vec{h}_q^{(s)}\}_{Q^m_{\leq d}}$ are the marginals computed with respect to $D_s$.
% , but ``should in the support of $\mathcal{D}^n_m$'' be in $(\{-1,1\}\times\{-1,1\})^m$}
In the following we show that if 
$\{\vec{u}_q\}_{q \in Q^m_{\leq d}}$
has tolerance $\mathsf{tol} = \frac{1}{n}$,
then with all but negligible probability over choice of 
$D_r \sim \mathcal{D}_m^n$,
$\max_{q \in Q^m_{\leq d}}\|\vec{h}_q^{(r)} - n \cdot \vec{u}_q\|_\infty \in
O(\ln^2(n) \cdot \sqrt{n})$.
This implies that the optimal $\{\vec{h}_q^{(s)}\}_{q \in Q^m_{\leq d}}$ must also satisfy
$\max_{q \in Q^m_{\leq d}}\|\vec{h}_q^{(s)} - n\cdot\vec{u}_q\|_\infty \in
O(\ln^2(n) \cdot \sqrt{n})$, which in turn implies that
$\max_{q \in Q^m_{\leq d}}\|\vec{h}_q^{(s)} - \vec{h}_q^{(r)}\|_\infty \in
O(\ln^2(n) \cdot \sqrt{n})$. Finally, the $\ell_1$ norm of any marginals can be bounded $\max_{q \in Q^m_{\leq d}}\|\vec{h}_q^{(s)} - \vec{h}_q^{(r)}\|_1 \in
O(2^d \cdot \ln^2(n) \cdot \sqrt{n})$.

We next show that  with all but negligible probability over choice of 
$D_r \sim \mathcal{D}_m^n$,
$\max_{q \in Q^m_{\leq d}}\|\vec{h}_q^{(r)} - n \cdot \vec{u}_q\|_\infty \in
O(\ln^2(n) \cdot \sqrt{n})$.
By Chernoff bounds and the tolerance guarantee, for a particular $q \in Q^m_{\leq d}$ and $\vec{t}\in\Omega_q$,
$\Pr[|n \cdot \vec{u}_q[\vec{t}] - \vec{h}_q^{(r)}[\vec{t}]| > \beta] \leq 2 \cdot \exp(-2(\beta-1)^2/n)$. We set 
$\beta = \ln^2(n) \cdot \sqrt{n}$
for this probability to be negligible in $n$.
Since we have also set parameters such that
$\sum_{q\in Q^m_{\leq d}}|\Omega_q| \leq n$
, after taking a union bound over all $q \in Q^m_{\leq d}$ and $\vec{t} \in \Omega_q$,
we have that
with all but negligible probability over choice of 
$D_r \sim \mathcal{D}_m^n$,
$\max_{q \in Q^m_{\leq d}}\|\vec{h}_q^{(r)} - n \cdot \vec{u}_q\|_\infty \in
O(2^d \cdot\ln^2(n) \cdot \sqrt{n})$.

Using the parameter settings in Definition~\ref{def:params} we invoke Theorem~\ref{thm:ASD_gen} to obtain the upper bound:

\begin{align*}
|L(\vec{w}_s,D_r) - L(\vec{w}_r,D_r)| &\in  O(\frac{1}{\sqrt{d}})
+ O\left(\frac{(K\tau\sqrt{m}+ \varphi(0)) \cdot (3m \cdot \max\{1,\tau\})^{d-1} \nu}{n}  \right)\\
&\in  O(\frac{1}{\sqrt{d}})
+ O \left ( \gamma (3m)^{d-1} \cdot \frac{2^d \cdot \ln^2(n) \cdot \sqrt{n}}{n} \right)\\
&\in  O(\frac{1}{\sqrt{d}})
+ O \left (\ln^{-1}(n)(6 \ln^5 n)^{d-1} \cdot \frac{2^d \cdot \ln^2(n) \cdot \sqrt{n}}{n} \right)\\
&\in  O\left(\sqrt{\frac{-\ln(\gamma)}{\gamma^{-2r/5}}} \right )
+ O \left ((\ln^5 n)^{d} \cdot \frac{12^d \cdot \ln(n) \cdot \sqrt{n}}{n} \right)\\
&\in  O\left(\sqrt{\frac{\ln(\ln(n))}{\ln(n)}}\right )
+ O \left (\frac{n^{1/4} \cdot \ln(n) \cdot \sqrt{n}}{n} \right)\\
&\in  O\left(\sqrt{\frac{\ln(\ln(n))}{\ln(n)}}\right )
+ O \left (\frac{ \ln(n)}{n^{1/4}} \right)\\
% &\in O\left(\sqrt{\frac{-\ln(\gamma)}{\gamma^{-2r/5}}} \right ) + O \left (\frac{\log^2(n)}{n^{1/4}} \right)\\
% &\in O\left(\sqrt{\frac{\ln(\ln(n))}{\ln(n)}}\right ) + O \left (\frac{\log^2(n)}{n^{1/4}} \right).
\end{align*}

\vspace{3mm}

We now turn to the lower bound (the second item in Theorem~\ref{th:lower_bound}).
For the lower bound, we utilize the following lower bound on the accuracy of non-adaptive statistical query algorithms, where the accuracy is measured by the \emph{classification error}: $\mathsf{err}_{f^*, \mathcal{D}_m}(\hat{f})\triangleq \Pr_{(\vec{x},y)\sim \mathcal{D}_m}[f^*(\vec{x}) \neq \hat{f}(\vec{x})]$. (Looking forward, we consider $\mathcal{D}_m$ being linearly separable, and $f^*$ is one of the linear separators. Therefore, $f^*(\vec{x}) = y$ for any $(\vec{x},y)$ in the support of $\mathcal{D}_m$.

\begin{theorem}[Theorem 5 in \cite{DF20}] \label{th:stat_query_lb}
Let $r \in (0,1)$, $\gamma \in (0, 2^{-1/(1-r)})$,
$m \geq 2 \cdot \gamma^{-2-2r/5}$
and define $\eta = \gamma^{1-r}$.
Let $\mathcal{A}$ be a non-adaptive statistical
query algorithm such that for any linear separator $f^*$
and distribution $\mathcal{D}_m$ over $X = \{-1, 1\}^{m}$ with margin
$\gamma(f^*, \mathcal{D}_m) \geq \gamma$, returns a hypothesis
$\hat{f}$ with $\mathbb{E}_{\mathcal{A}}[\mathsf{err}_{f^*, \mathcal{D}_m}(\hat{f})] \leq 1/2 - \eta$.
If $\mathcal{A}$ has access to statistical queries with tolerance $\mathsf{tol} \geq \exp(-c\gamma^{-2r/5})$, then 
$\mathcal{A}$ requires at least $\exp(c\gamma^{-2r/5})$
queries, where $c > 0$ is a constant depending only on $r$.
\end{theorem}

\begin{corollary} \label{cor:lower_bound_fixed_params}
For the parameter settings given in Definition~\ref{def:params}, for any (even computationally inefficient) algorithm $\mathsf{Syn}$,
the algorithm defined in Algorithm~\ref{fig:non_adap_stat_query_alg} has error at least 
$\mathbb{E}[\mathsf{err}_{f^*, D}(\hat{f})] > 1/4$.
\end{corollary}

The corollary follows by noting that for the parameter settings of
$r, n, m, \gamma, d$ in Definition~\ref{def:params}, all of the following hold:
$r \in (0,1)$, $\gamma \in (0, 2^{-1/(1-r)})$,
$m \geq 2 \cdot \gamma^{-2-2r/5}$,
$\eta = \gamma^{1-r} \leq \frac{1}{4}$,
$\mathsf{tol} = \frac{1}{n} \geq \exp(-c\gamma^{-2r/5})$,
and $\sum_{q\in Q^m_{\leq d}}|\Omega_q| \leq 2^d \cdot m^d <e^ {8c'\gamma^{-2r/5}}\leq \exp(c\gamma^{-2r/5})$.
Since the algorithm defined in Algorithm~\ref{fig:non_adap_stat_query_alg} is a non-adaptive statistical query algorithm with tolerance $\mathsf{tol} = \frac{1}{n}$ and making less than $\exp(c\gamma^{-2r/5})$ number of statistical queries, Theorem~\ref{th:stat_query_lb} implies that
its error must be at least $1/4$.

The above corollary gives a bound on the error of linear separator $\vec{w}_s$ outputted by Algorithm~\ref{fig:non_adap_stat_query_alg}, whereas we need a bound on the difference in loss between $\vec{w}_s$ and the optimal linear separator. The following Claim allows us to relate the error and the loss.
\begin{claim} \label{claim:accuracy_loss}
Let the loss function $L''(\vec{w},\mathcal{D}_m) := \mathbb{E}_{(\vec{x}, y) \sim \mathcal{D}_m}[\gamma \cdot \varphi_\gamma(y \langle \vec{w}, \vec{x} \rangle )]$, where the expectation is taken with respect to distribution $\mathcal{D}_m$.
Let $\hat{\vec{w}}$ be any vector of norm at most $\tau$.
Let $\vec{w}^*$ be the optimal linear separator with respect to $L''(\vec{w},\mathcal{D}_m)$.

Let $A$ be any algorithm.
If $\mathbb{E}_{\hat{\vec{w}} \leftarrow A}[L''(\hat{\vec{w}},\mathcal{D})] \leq L''(\vec{w}^*,\mathcal{D}) + \frac{\gamma}{8}$, then
$\mathbb{E}_{\hat{\vec{w}} \leftarrow A}[\mathsf{err}_{\mathcal{D}_m}(\hat{\vec{w}})] \leq 1/4$.
\end{claim}

\begin{proof}
Assume $\mathbb{E}_{\hat{\vec{w}} \leftarrow A}[L''(\hat{\vec{w}},\mathcal{D})] \leq L''(\vec{w}^*,\mathcal{D}) + \frac{\gamma}{8}$. Then this implies that
$\mathbb{E}_{\hat{\vec{w}} \leftarrow A}[L'(\hat{\vec{w}},\mathcal{D}_m)] \leq L'(\vec{w}^*,\mathcal{D}_m) + \frac{1}{8}$,
where $L'_{\mathcal{D}_m}$ is the cost function $L'(\vec{w},\mathcal{D}_m) := \mathbb{E}_{(\vec{x}, y) \sim \mathcal{D}_m}[\varphi_\gamma(y \langle \vec{w}, \vec{x} \rangle )]$.
By Claim 3 in \cite{DF20}, this implies that
$\mathbb{E}_{\hat{\vec{w}} \leftarrow A}[\mathsf{err}_{\mathcal{D}_m}(\hat{\vec{w}})] \leq 1/4$.
\end{proof}

Taking Corollary~\ref{cor:lower_bound_fixed_params} and Claim~\ref{claim:accuracy_loss} together, 
we have that for every algorithm $\mathsf{Syn}$
there exists a distribution $\mathcal{D}_m$
and a set of vectors $\{\vec{u}_q\}_{q \in Q^m_{\leq d}}$ of tolerance $\mathsf{tol} = \frac{1}{n}$ such that
\begin{equation} \label{eq:loss_bound}
\mathbb{E}_{D_s \leftarrow \mathsf{Syn}(n,  \{n \cdot\vec{u}_q\}_{q \in Q^m_{\leq d}})}[L''(\vec{w}_s,\mathcal{D}_m)] \geq L''(\vec{w}^*,\mathcal{D}_m) + \frac{\gamma}{8}.
\end{equation}
We must now convert the expected loss given above to \emph{excess empirical risk w.r.t.~the real training data}.
To do so, we note that
for every $(\vec{x}, y)$ in the support of $\mathcal{D}_m$,
$\varphi(y \langle \vec{w}, \vec{x} \rangle)$
is lower bounded by $0$ and upper bounded by $1$ and therefore so is
$\mathbb{E}_{D_s \leftarrow \mathsf{Syn}(n,  \{n \cdot\vec{u}_q\}_{q \in Q^m_{\leq d}})}[\varphi(y \langle \vec{w}_s, \vec{x} \rangle)]$.
Recall that
\begin{equation} \label{eq:hoeff_1}
L''(\vec{w}^*,\mathcal{D}_m)
= \mathbb{E}_{(\vec{x},y) \leftarrow \mathcal{D}_m}[\varphi(y \langle \vec{w}^*, \vec{x} \rangle)].
\end{equation}
By linearity of expectation, we also have that
\begin{equation} \label{eq:hoeff_2}
\mathbb{E}_{D_s \leftarrow \mathsf{Syn}(n, \{n \cdot\vec{u}_q\}_{q \in Q^m_{\leq d}})}[L''(\vec{w}_s,\mathcal{D}_m)]
= \mathbb{E}_{(\vec{x},y) \leftarrow \mathcal{D}_m}[\mathbb{E}_{D_s \leftarrow \mathsf{Syn}(n, \{n \cdot\vec{u}_q\}_{q \in Q^m_{\leq d}})}[\varphi(y \langle \vec{w}_s, \vec{x} \rangle)]].
\end{equation}
Since our setting of parameters implies that $\frac{n}{\ln^2(n)} \geq \frac{800}{\gamma^2}$,
we have by (\ref{eq:hoeff_1}), (\ref{eq:hoeff_2}) and by
standard Hoeffding bounds that with all but negligible probability over choice of
$D_r$,
\begin{equation} \label{eq:hoeffding_loss}
\mathbb{E}_{D_s \leftarrow \mathsf{Syn}(n, \{n \cdot\vec{u}_q\}_{q \in Q^m_{\leq d}})}[L''(\vec{w}_s,\mathcal{D}_m)] - \mathbb{E}_{D_s \leftarrow \mathsf{Syn}(n, \{n \cdot\vec{u}_q\}_{q \in Q^m_{\leq d}})}[L(\vec{w}_s,D_r)] \leq \frac{\gamma}{20} \quad \mbox{ and } \quad
L(\vec{w}^*,D_r) - L''(\vec{w}^*,\mathcal{D}_m) \leq \frac{\gamma}{20}.
\end{equation}
Therefore, combining (\ref{eq:loss_bound}), (\ref{eq:hoeffding_loss}), and by the optimality of $\vec{w}_r$,
\begin{equation} \label{eq:loss_gamma}
\mathbb{E}_{D_s \leftarrow \mathsf{Syn}(n, \{n \cdot\vec{u}_q\}_{q \in Q^m_{\leq d}})}[L(\vec{w}_s,D_r)] \geq L(\vec{w}^*,D_r) + \frac{\gamma}{16} \geq L(\vec{w}_r,D_r) + \frac{\gamma}{16}.
\end{equation}
Substituting $\gamma = \frac{1}{\ln^3(n)}$ into (\ref{eq:loss_gamma}) we obtain
\[
|\mathbb{E}_{D_s \leftarrow \mathsf{Syn}(n, \{n \cdot\vec{u}_q\}_{q \in Q^m_{\leq d}})}[L(\vec{w}_s,D_r)] - L(\vec{w}_r,D_r)| \in \Omega(\frac{1}{\ln^3(n)}),
\]
which concludes the proof of the theorem.
\end{proof}

\section{More on Synthetic Data Generation}
\label{app:syn_gen}
\subsection{Private-PGM} 
\label{sec:private_PGM}
The core of Private-PGM is to fit a graphical model to the sensitive data in a differentially-private way, and then use the graphical model to generate the synthetic data.
The high-level steps involve computing noisy marginals of the sensitive data for selected sets of attributes of small size. Secondly, executing an optimization problem to identify a probability distribution that "best explains" these noisy marginal measurements, representing it as a probabilistic graphical model. Finally, generate synthetic data that closely matches the estimated distribution. 
Please refer to Algorithm~\ref{alg:PGM_data} for the pseudocode for generating synthetic data using Private-PGM.

\ifsubmissionshort{
\begin{algorithm}
   \caption{$f_\mathsf{PPGM}$ Generating Synthetic Data using Private PGM~\citep{DBLP:journals/corr/abs-2108-04978}}
    \label{alg:PGM_data}
\begin{algorithmic}
   \STATE {\bfseries Input:} Real dataset $D_{r}\in \mathbb{R}^{n \times (m+1)}$,  marginals queries $Q$, noise scale $\sigma$
  \STATE {\bfseries Output:} Synthetic Dataset $D_{s}$
   % \STATE {\bfseries Measure Noise Marginals:}
   \FOR{$q\in Q$}
   \STATE Measuring marginal: $\vec{h}_q=M_q(D_{r})$, where $M_q$ is the algorithm for measuring marginals;
   \STATE {\bfseries Add noise:} $\vec{\hat{h}}_q=\vec{h}_q+\mathcal{N}(0, \sigma)$;
   \ENDFOR
   
   % \STATE {\bfseries Synthesize data:}
   \STATE {\bfseries Generate graphical model} $P_\theta$ with weight vector $\theta$: $ \operatorname{argmin}_\theta \sum_{q\in Q} \Norm{M_q(P_\theta)-\vec{\hat{h}}_q}^2_2$ ;
\STATE {\bfseries Generate synthetic data} $D_{s}$ using $P_\theta$ using Algorithm~\ref{alg:synthetic_data} and Algorithm~\ref{alg:synthetic_culumn}
   (See Below);
\end{algorithmic}
\end{algorithm}
}

\else{
\begin{algorithm}
\caption{$f_\mathsf{PPGM}$ Generating Synthetic Data using Private PGM}
\label{alg:PGM_data}
\DontPrintSemicolon
  \KwIn{Real dataset $D_{r}\in \mathbb{R}^{n \times (m+1)}$,  marginals queries $Q$, noise scale $\sigma$}
  \KwOut{Synthetic Dataset $D_{s}$}

\BlankLine
  \textbf{Measure Noise Marginals:} \;
  
  \For{$q\in Q$}{Measuring marginal: $\vec{h}_q=M_q(D_{r})$, where $M_q$ is the algorithm for measuring marginals.\;
  \textcolor{red}{Add noise: $\vec{\hat{h}}_q=\vec{h}_q+\mathcal{N}(0, \sigma)$}
}
\BlankLine
  \textbf{Synthesize data}\;{
     i). Generate graphical model $P_\theta$ with weight vector $\theta$: $ \operatorname{argmin}_\theta \sum_{q\in Q} \Norm{M_q(P_\theta)-\vec{\hat{h}}_q}^2_2$ \;
ii). Generate synthetic data $D_{s}$ using $P_\theta$ using Algorithm~\ref{alg:synthetic_data} and Algorithm~\ref{alg:synthetic_culumn}
  } (See Appendix~\ref{app:algorithem})  
\end{algorithm}
}\fi

The algorithm~\ref{alg:PGM_data} above makes use of the following two subroutines to generate synthetic data from the graphical model: algorithm~\ref{alg:synthetic_data} and algorithm~\ref{alg:synthetic_culumn}.

\ifsubmissionshort{
\begin{algorithm}
\caption{Synthetic data generation}
\label{alg:synthetic_data}
\begin{algorithmic}
   \STATE {\bfseries Input:} graphical model (see Algorithm~\ref{alg:PGM_data})
   \STATE {\bfseries Output:} dataset (synthetic dataset)
   
   \STATE Initialize the set of processed attributes to the empty set;
   \FOR{each attribute $i$}
   \STATE Let C be the set of all neighbors of $i$ in the graphical model, intersected with the set of
processed attributes;
\STATE Group data by C, and
\FOR{each group in C}
   \STATE Calculate $\mu$ from the graphical model, the vector of fractional counts for every possible value of attribute $i$, for the given group of other attributes;
   \STATE Generate synthetic column for this group using Algorithm~\ref{alg:synthetic_culumn};
   \STATE Add this partial column to the grouped rows in the dataset;
   \ENDFOR
   \STATE Add $i$ to the set of processed attributes;
   \ENDFOR       
\end{algorithmic}
\end{algorithm}

\begin{algorithm}
\caption{Synthetic column}
\label{alg:synthetic_culumn}
\begin{algorithmic}
   \STATE {\bfseries Input:} $\mu$ (vector of fractional counts), $n$ (total number of samples to generate)
   \STATE {\bfseries Output:} column (synthetic column of data)
  \STATE Generate $\lfloor \mu_t \rfloor$ items with value t and add to column for each t in domain;
 \STATE Calculate remainders: $p_t = \mu_t-\lfloor \mu_t \rfloor $ ;
  \STATE Sample $n -\sum_t \lfloor \mu_t \rfloor$ items (without replacement) from distribution proportional to $p_t$, and add to column;
  \STATE Shuffle values in column;
  
\end{algorithmic}
\end{algorithm}
}
\else{
\begin{algorithm}[H]
\caption{Synthetic data generation}
\label{alg:synthetic_data}
\DontPrintSemicolon
  \KwIn{graphical model (see Algorithm~\ref{alg:PGM_data})} 
  \KwOut{dataset (synthetic dataset)}

\BlankLine
  Initialize the set of processed attributes to the empty set\;
  \For{each attribute $i$}{Let C be the set of all neighbors of $i$ in the graphical model, intersected with the set of
processed attributes\;
  Group data by C, and \For {each group in C}{
  Calculate $\mu$ from the graphical model, the vector of fractional counts for every possible
value of attribute $i$, for the given group of other attributes\;
Generate synthetic column for this group using Algorithm \ref{alg:synthetic_culumn}\;
Add this partial column to the grouped rows in the dataset\;
  } 
  Add $i$ to the set of processed attributes\;}
\end{algorithm}

\begin{algorithm}[H]
\caption{Synthetic column}
\label{alg:synthetic_culumn}
\DontPrintSemicolon
  \KwIn{$\mu$ (vector of fractional counts), $n$ (total number of samples to generate)}
  \KwOut{column (synthetic column of data)}

\BlankLine
  Generate $\lfloor \mu_t \rfloor$ items with value t and add to column for each t in domain\;
  Calculate remainders: $p_t = \mu_t-\lfloor \mu_t \rfloor $ \;
  Sample $n -\sum_t \lfloor \mu_t \rfloor$ items (without replacement) from distribution proportional to $p_t$, and add to column\;
  Shuffle values in column\;
\end{algorithm}
}
\fi

In the Private-PGM approach, differential privacy is achieved by applying a noise mechanism to the marginal measurements.
% , which is shown as \ifsubmissionshort{\textit{Add noise} }\else{Step 4 }\fi in Algorithem~\ref{alg:PGM_data}.  
In our experiments, we use the Gaussian mechanism~\citep{dwork2014algorithmic}, as it is reliable and widely used noise mechanisms for enforcing differential privacy. 
For any single individual's data is altered, it can affect up to two queries by 1 in each marginal measurement.
This results in an sensitivity  of $\sqrt{2|Q|}$ for all measurements. 
Invoking Theorem~\ref{thm:gauss}, adding a Gaussian noise to each query with variance,
\ifsubmissionshort{$\sigma^2=\frac{2\sqrt{2|Q|}^2\log(1.25/\delta)}{\epsilon^2}$, }
\else{
\begin{align} \label{eq:simga}
\sigma^2=\frac{2\sqrt{2|Q|}^2\log(1.25/\delta)}{\epsilon^2}, 
\end{align} }\fi
we have that the collection of noisy marginals outputted in step
\ifsubmissionshort{\textit{Add noise}}
\else{ 4 }
\fi
of Algorithm 3 achieves $(\epsilon, \delta)$-differential privacy. Since the inputs to 
\ifsubmissionshort{\textit{Generate graphical model}}
\else{step 5} \fi
of Algorithm 3 are differentially private, then the synthetic data finally outputted by Algorithm 3 must also be $(\epsilon, \delta)$-differentially private.
Any subsequent analyses, including the model training on synthetic data, and further analysis using the trained model, are considered as post-processing. According to Theorem~\ref{thm:post_proc}, these analyses will continue to uphold $(\epsilon, \delta)$-DP.

\subsection{AIM}
\label{sec:AIM}
The optimal choice of marginal queries/attribute sets to be captured by the synthetic data can be difficult to determine, and can itself leak private information. Therefore, AIM uses an adaptive and iterative algorithm to ``automatically'' select marginal query that best reduces the distance between the real and synthetic data. 

More specifically, AIM allows the user to pre-specify a privacy budget $\rho$ and a collection $Q$ of marginal queries to be selected from. For instance, $Q$ can be the collection of all 3-order marginal queries. 
The algorithm starts with an initial synthetic data distribution $\mathcal{\hat{D}}_0$.
In each iteration $i=1,2,\dots$, it randomly selects a marginal query $q_i$ from $Q$ with probability proportional to $q_i$'s \emph{quality score} that captures the distance between its real marginal and its marginal evaluated from the current estimated synthetic data distribution $\mathcal{\hat{D}}_{i-1}$.
This randomness in the selection process ensures differential privacy and the method is formally known as the exponential mechanism \citep{MT07} in DP literature.
Then, AIM uses the Gaussian mechanism to measure the marginal of the selected query, followed by using Private-PGM to estimate data distribution $\mathcal{\hat{D}}_{i}$ from all noisy marginals measured so far. 
Finally, to terminate, AIM keeps track of the privacy parameter and the junction tree size corresponding to the selected marginals and makes sure they do not exceed their limits.

To handle composition easily, AIM uses zero-concentrated differential privacy (zCDP) and formally claims the following theorem. 
\begin{theorem}
     For any $T \geq m$, where $T$ is a user-specified limit on the number of iterations, and $\rho \geq 0$, AIM satisfies $\rho$-zCDP.
\end{theorem}

This can be converted to the standard DP guarantee using the following proposition:
\begin{proposition}[zCDP to DP~\citep{zCDP_to_DP}] If a mechanism $M$ satisfies $\rho$-zCDP, it also satisfies $(\epsilon, \delta)$-differential privacy for all $\epsilon\geq 0$ and
$\delta=\min_{\alpha>1}\frac{\exp((\alpha-1)(\alpha \rho-\epsilon)}{\alpha-1}(1-\frac{1}{\alpha})^\alpha$.   
\end{proposition}

While AIM algorithm may only select a small subset of marginal queries to measure before termination, it provides upper bounds of the $\ell_1$ difference on both the selected marginals and non-selected marginals in $Q$. The former can be easily derived as the selected marginal are measured with Gaussian noise. For the latter, it utilizes the relation between the last selected marginal query and the remaining non-selected ones. In particular, as the marginal query is selected with probability proportional to the exponential of their marginal distance to the real ones, this provides a way to derive the upper bound on all remaining non-selected marginals.
More formally, for a marginal query $q\in Q$, let $n_q = |\Omega_q|$, and $w_q$ be a parameter that specifies the ``importance'' of $q$ among $Q$, which is larger if the average intersection size of $q$ with other sets in $Q$ is high).
At $i$-th iteration, let $\sigma_i, \epsilon_i$ be the hyperparameters that AIM automatically selected to determine the amount of noise, and $q_i$ be the marginal query selected at this iteration, and $Q_i \subseteq Q$ is the marginal queries that can be selected from, which only includes marginal queries that can be measured without significantly increase the junction tree size for Private-PGM. AIM paper proves the following theorem:

\begin{theorem}[Confidence Bound for Non-selected Marginal Query]Let $\Delta_i = \max_{q \in Q_i} w_q$. For all $q\in Q_i$, with probability at least $1-e^{-\lambda_1^2/2}-e^{-\lambda_2}$:
\begin{align*}
    \|\vec{h}_q^{(r)}-M_q(\mathcal{\hat{D}}_{i-1})\|_1 \leq w_q^{-1}(B_q + \lambda_1\sigma_i \sqrt{n_{q_i}} + \lambda_2 \frac{2\Delta_i}{\epsilon_i}),
\end{align*}
    where $B_q$ is equal to:
\begin{align*}
    w_{q_i}\|M_q(\mathcal{\hat{D}}_{t-1})-\vec{h}_{q_i}^{(r)}\|_1 + \sqrt{2/\pi} \sigma_i (w_qn_q-w_{q_i}n_{q_i}) + \frac{2\Delta_i}{\epsilon_i} \log(|Q_i|)
\end{align*}
\label{thm:AIM bound non-selected marginal}
\end{theorem}

In McKenna et al’s empirical evaluation, it indicated that the marginal selection approach employed by AIM makes it consistently outperformed all other marginal-preserving mechanisms for preserving statistical properties. 
In Section
\ifsubmissionshort{~\ref{app: Assess AIM for Different Classifiers}}\else{~\ref{sec: Assess AIM for Different Classifiers}}\fi, we will show our experimental results that extend this advantage to consistently learning multiple models with different classifiers.

\section{Additional Experimental Results}
\label{app:experiments}
\subsection{Performance of synthetic data with various privacy budgets}
\label{app: Performance of synthetic data with various privacy budgets}
This section provides Table~\ref{tab:performance_in_diff_eps}, serving as a supplement to Section~\ref{sec:experiment Synthetic data in varied privacy budgets}, presenting numerical test results of performance for various evaluation metrics across six datasets, employing varied $\epsilon$, using
AIM synthetic data generation. Refer to Figure~\ref{fig:performance_in_diff_eps} for a visual representation.

\begin{table*} [tb]
\caption{Table presenting comprehensive performance results for various evaluation metrics across six datasets, employing varied $\epsilon$, using AIM synthetic data generation. Refer to Figure~\ref{fig:performance_in_diff_eps} for a visual representation.  }
\vskip 0.1in
 \centering
\begin{tabular}{lcccccccccc}
\toprule
\multicolumn{2}{l}{Dataset}  & \multicolumn{8}{c}{Synthetic data with varied epsilon} &Real data\\
\cmidrule(lr){3-10} \cmidrule(lr){11-11}
& & 0.25 & 0.5 & 0.75 & 1 & 1.25 & 1.5 & 1.75 & 2 & \\
\midrule
Adult & Accuracy & 0.832 & 0.836 & 0.838 & 0.84 & 0.84 & 0.841 & 0.842 & 0.842 & 0.843  \\
 &ROC score & 0.879 & 0.883 & 0.886 & 0.887 & 0.889 & 0.888 & 0.889 & 0.889 & 0.891 \\
 &Empirical Risk & 0.358 & 0.353 & 0.349 & 0.347 & 0.345 & 0.345 & 0.344 & 0.344 & 0.34 \\
\midrule
Churn & Accuracy & 0.719 & 0.742 & 0.748 & 0.748 & 0.748 & 0.751 & 0.756 & 0.757 & 0.766 \\
&ROC score & 0.762 & 0.789 & 0.801 & 0.801 & 0.801 & 0.804 & 0.809 & 0.81  & 0.826\\
&Empirical Risk & 0.546 & 0.52  & 0.506 & 0.505 & 0.498 & 0.501 & 0.493 & 0.491 & 0.47   \\
\midrule
Compas & Accuracy & 0.587 & 0.645 & 0.646 & 0.658 & 0.663 & 0.663 & 0.664 & 0.666 & 0.671 \\
& ROC score & 0.6   & 0.68  & 0.685 & 0.705 & 0.709 & 0.71  & 0.71  & 0.71  & 0.718  \\
& Empirical Risk & 0.674 & 0.638 & 0.636 & 0.624 & 0.625 & 0.624 & 0.623 & 0.622 & 0.617  \\
\midrule
Dutch & Accuracy & 0.812 & 0.815 & 0.817 & 0.816 & 0.817 & 0.817 & 0.817 & 0.817 & 0.817 \\
& ROC score &  0.884 & 0.885 & 0.887 & 0.886 & 0.886 & 0.886 & 0.886 & 0.886 & 0.886 \\
& Empirical Risk & 0.43  & 0.427 & 0.424 & 0.426 & 0.425 & 0.426 & 0.426 & 0.426 & 0.427\\
\midrule
Heart & Accuracy & 0.572 & 0.675 & 0.701 & 0.719 & 0.749 & 0.756 & 0.767 & 0.78  & 0.802 \\
& ROC score & 0.569 & 0.728 & 0.778 & 0.78  & 0.82  & 0.834 & 0.838 & 0.853 & 0.883 \\
& Empirical Risk & 0.986 & 0.699 & 0.617 & 0.593 & 0.536 & 0.508 & 0.496 & 0.484 & 0.452 \\
\midrule
Law & Accuracy & 0.892 & 0.894 & 0.895 & 0.895 & 0.895 & 0.894 & 0.895 & 0.896 & 0.901 \\
& ROC score & 0.829 & 0.854 & 0.856 & 0.857 & 0.859 & 0.859 & 0.857 & 0.858 & 0.869 \\
& Empirical Risk & 0.274 & 0.259 & 0.257 & 0.256 & 0.254 & 0.258 & 0.257 & 0.255 & 0.245 \\
\bottomrule
\end{tabular}
\label{tab:performance_in_diff_eps}
\vskip -0.1in
\end{table*}
}

\subsection{Comparison with Other DPML Techniques}
\label{app: Comparison with Other DPML Techniques}
To appraise the performance of AIM synthetic data in comparison to prevailing DP-ML approaches, we conducted training using two DP-ML methods. The first one is Differentially-Private Stochastic Gradient Descent (DP-SGD)~\citep{Abadi_2016}, which ensures differential privacy by introducing carefully calibrated noise to the gradients during the training process. Refer to Algorithm~\ref{alg:DP-SGD} for details. 
% We set the noise scale as described in~\citep{yu2021gradient}. 
The second method, Private Aggregation of Teacher Ensembles (PATE) learning method~\citep{papernot2017semisupervised},
assumes a slightly different threat model, as we discuss next. 
The PATE method entails training multiple teacher models on sensitive training data and ensuring differential privacy by introducing noise to the counts of teacher predictions for each subsequent query made.
In addition, a \emph{public}, unlabeled training dataset is required, and differentially-private queries to the teachers are used to label the data.  Finally, a student 
model is trained using the newly labeled data, and this student model can then be released as the final DP-ML model.
Note that the model that is ultimately released does not preserve the privacy of the unlabeled training dataset.
Thus, this mechanism crucially assumes existence of public, unlabeled training data. Therefore, in order to compare against PATE we construct 3 datasets, a private-labeled-training dataset for teacher models, a public-unlabeled-training dataset for student model, and a testing dataset to assess the performance of the student model. We provide additional details below.

\ifsubmissionshort{
\begin{algorithm}
   \caption{Differentially-Private Stochastic Gradient Descent (DP-SGD)~\citep{iyengar2019towards} Algorithm 2}
    \label{alg:DP-SGD}
\begin{algorithmic}
   \STATE {\bfseries Input:} Training dataset: $(X, Y)\in D$, where features $X\in \mathbb{R}^{n\times m}$, labels $Y\in \mathbb{R}^n$, Lipschitz constant: L,
privacy parameters: $(\epsilon, \delta)$, number of
iterations: T, minibatch size: B, learning rate: $\eta$, gradient norm bound $C$. 
  \STATE {\bfseries Output:} Logistic Regression Model with weights $\vec{w}$
  \STATE Initialize weights $\mathbf{w}=\{0\}^m$;
  \STATE $\sigma^2=\frac{16L^2T \log{(1/\delta)}}{n^2\epsilon^2}$;
   \FOR{$t\in T$}
   \STATE Sample $B$ samples uniformly with replacement from $D$: $(x_1, y_1), ..., (x_B, y_B)$;
   \STATE {\bfseries Clip gradient:} $\hat{\nabla} L(x_i,y_i)=\nabla L(x_i,y_i)/\max{(1, \frac{\Norm{\nabla L(x_i,y_i)}_2}{C})}$
   \STATE {\bfseries Add noise:} $\nabla L_t(\vec{w})=\frac{1}{B}\sum_{i=1}^B \hat{\nabla} L(x_i,y_i)+\mathcal{N}(0, \sigma)$;
    \STATE Update weights: $\mathbf{w} = \mathbf{w} - \eta \cdot \nabla L_{t}(\vec{w})$;
   \ENDFOR
      
\end{algorithmic}
\end{algorithm}
}

\else{
\begin{algorithm}
\caption{Differentially-Private Stochastic Gradient Descent (DP-SGD)}
\label{alg:DPSGD}
\DontPrintSemicolon
  \KwIn{Training dataset: $(X, Y)\in D$, where features $X\in \mathbb{R}^{n\times m}$, labels $Y\in \mathbb{R}^n$, Lipschitz constant: L,
privacy parameters: $(\epsilon, \delta)$, number of
iterations: T, minibatch size: B, learning rate: $\eta$,  gradient norm bound $C$. }

\BlankLine
Initialize weights $\mathbf{w}={0}^m$ \;
$\sigma^2=\frac{16L^2T \log{(1/\delta)}}{n^2\epsilon^2}$\;
\For{$t\in T$}{
Sample $B$ samples uniformly with replacement from $D$: $(x_1, y_1), ..., (x_B, y_B)$\;
\textbf{Clip gradient:} $\hat{\nabla} L(x_i,y_i)=\nabla L(x_i,y_i)/\max{(1, \frac{\Norm{\nabla L(x_i,y_i)}_2}{C})}$\;
\textbf{Add noise:} $\nabla L_t(\vec{w})=\frac{1}{B}\sum_{i=1}^B \hat{\nabla} L(x_i,y_i)+\mathcal{N}(0, \sigma)$\;
Update weights: $\mathbf{w} = \mathbf{w} - \eta \cdot \nabla L_{t}(\vec{w})$\;
}
  
\end{algorithm}
}
\fi

In our experiments, we retained the standard procedure of splitting the real data into 80\% training set and 20\% testing set, a consistent approach across all three DP-ML methods. For PATE-learning, we additionally sampled 100 data points from the training data (20 data points for the Heart Data,  due to its small dataset size), corresponding to the public, unlabeled data, and set those aside for later training of the student model.  The teachers models were trained on the remaining training set using the scikit-learn's logistic regression model with the LBFGS solver, the same algorithm used for training the student model. In sum, all three methods end up outputting a DP-ML model, and they all preserve the DP of training data, while PATE has 100 data points less in its training data, and we evaluate performance for all models using the testing data.

Refer to Figure~\ref{fig: aim vs dpgd vs pate} for the detailed parameters setup, which also displays the accuracy comparison among the three methods across six datasets. 
We note that the AIM and PATE models were trained using second-order methods such as Newton's method, converge faster, as opposed to gradient descent used by DP-SGD.
Secondly, we notes that the quality of the model obtained from DP-SGD for some dataset, i.e. Heart and Dutch datasets, is less competitive.
We believe it may be possible to further improve the quality of the model outputted by DP-SGD but it would require a 
considerable amount of effort in tuning its essential hyperparameters, such as learning rate, iterations and decay rate.  We further note that such fine-tuning incurs its own privacy leakage resulting from either running multiple differentially-private training runs to set the hyperparameters, or from setting hyperparameters based on \textit{non-private} training runs~\citep{tunning_hyperparameter}.

In summary, in our experiments, under identical privacy budgets, $\epsilon$, the Pre-DPML approach with AIM-generated synthetic data yielded a model that performs as well as or better than the models generated via the two Training-DPML methods, with the added benefit that with the Pre-DPML approach subsequent training can be performed on the synthetic data without increasing the privacy budget.

\begin{figure*}[tb]
\centering
\vskip 0.1in
\centering
\hspace{-3em}
\begin{minipage}{.2\textwidth}
\begin{subfigure}[Adult Dataset]
{\includegraphics[scale=0.25]{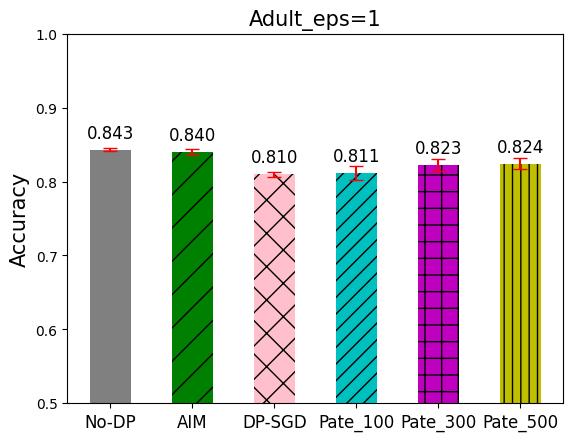}}
\label{fig:adult_DPMLs}
\end{subfigure}
\end{minipage}
\hspace{6em}
\begin{minipage}{.22\textwidth}
\begin{subfigure}[Compas Dataset]
{\includegraphics[scale=0.25]{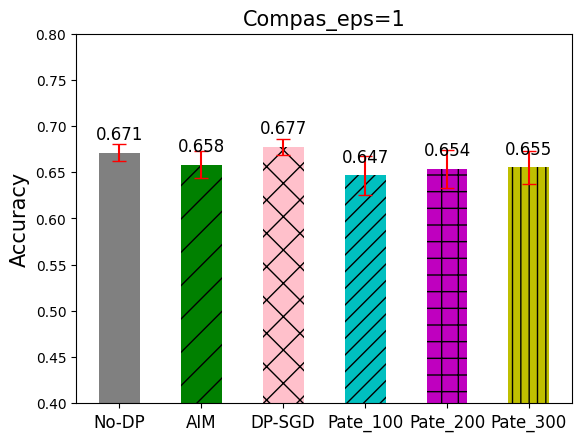}}
\label{fig:compas_DPMLs}
\end{subfigure}
\end{minipage}
\hspace{5em}
\begin{minipage}{.22\textwidth}
\begin{subfigure}[Law Dataset]
{\includegraphics[scale=0.25]{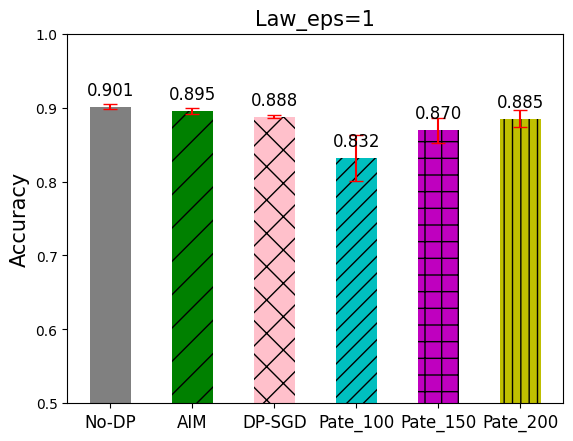}}
\label{fig:law_DPMLs}
\end{subfigure}
\end{minipage}

\vspace{0.5cm} 
\hspace{-3em}
\begin{minipage}{.22\textwidth}
\begin{subfigure}[Heart Dataset]
{\includegraphics[scale=0.25]{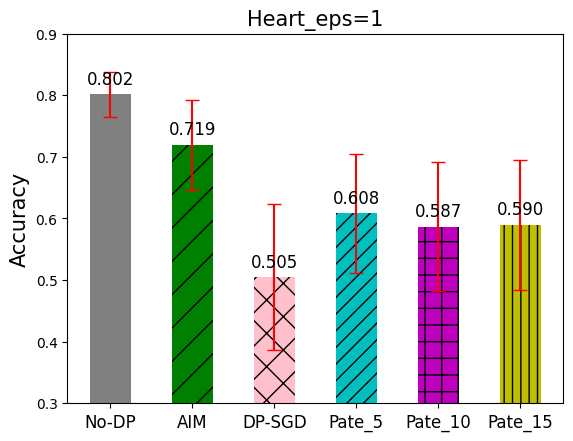}}
\label{fig:heart_DPMLs}
\end{subfigure}
\end{minipage}
\hspace{5em}
\begin{minipage}{.22\textwidth}
\begin{subfigure}[Dutch Dataset]
{\includegraphics[scale=0.25]{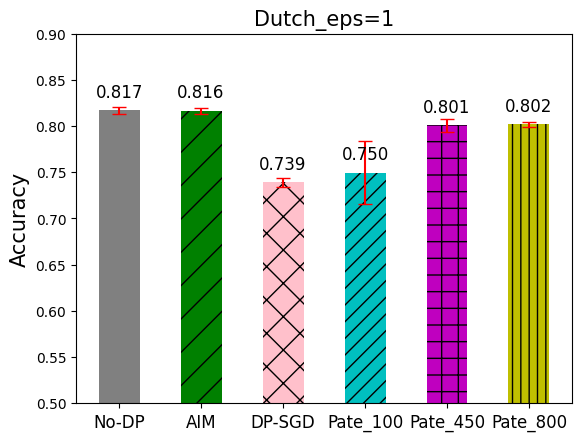}}
\label{fig:dutch_DPMLs}
\end{subfigure}
\end{minipage}
\hspace{5em}
\begin{minipage}{.22\textwidth}
\begin{subfigure}[Churn Dataset]
{\includegraphics[scale=0.25]{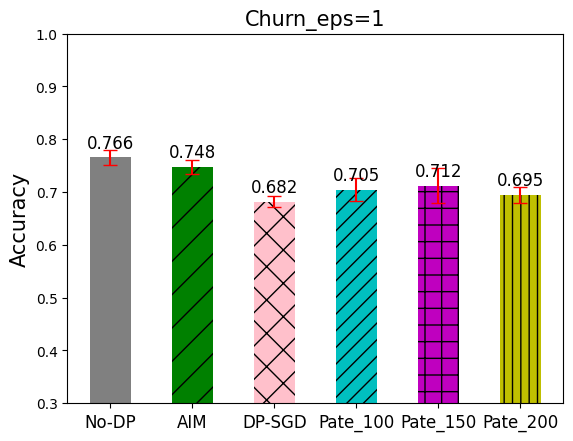}}
\label{fig:churn_DPMLs}
\end{subfigure}
\end{minipage}
\caption{We train the six dataset with DP-SGD approach that was described as Algorithm~\ref{alg:DP-SGD}, incorporating a gradient norm clipping threshold as 1, and differential privacy budget, epsilon=1. Specifically, we select the
learning rate from \{1, 5\},
running step T from \{300, 500, 1000\}, decay rate from \{0.1, 0.5\}, and batch size from\{20, 100, 200, 500, 1000, 3000\}.  Additionally, we train another DP method, PATE-learning, based on \citep{papernot2017semisupervised}. For each dataset, we consider three different teacher numbers chosen from  \{10, 15, 20, 100, 150, 200, 300, 450, 800\}. The figure illustrates a comparison of accuracy using various differential privacy methods, which includes Non-DP, AIM (generated DP synthetic data), DP-SGD, PATE learning (with 3 teacher numbers), respectively.}
\label{fig: aim vs dpgd vs pate}
\vskip -0.1in
\end{figure*}

\subsection{Assess AIM for Different Classifiers}
\label{app: Assess AIM for Different Classifiers}
We proposed AIM as the tool to generate synthetic data. Here we would show why select smartly marginal using AIM mechanism is beneficial for generating synthetic data. Figure \ref{fig: private-pgm vs aim} shows the experiments we conducted on three(3) datasets. For each dataset, we generated synthetic data with $\epsilon=1$, using AIM that using exponential mechanism to select the most useful marginals. 
We trained three classification models with two different target labels, \{$y_1, y_2, y_3$\}. 
The result reveals that, the performance of classifiers  trained on real data and AIM data are comparable.   
This suggests that AIM is effective even without prior knowledge and maintains its performance across various classifier. 
            
\begin{figure*}[tb]
\vskip 0.1in
\centering
\hspace{-5em}
\begin{minipage}{.22\textwidth}
\begin{subfigure}[Dutch Dataset]
{\includegraphics[scale=0.33]{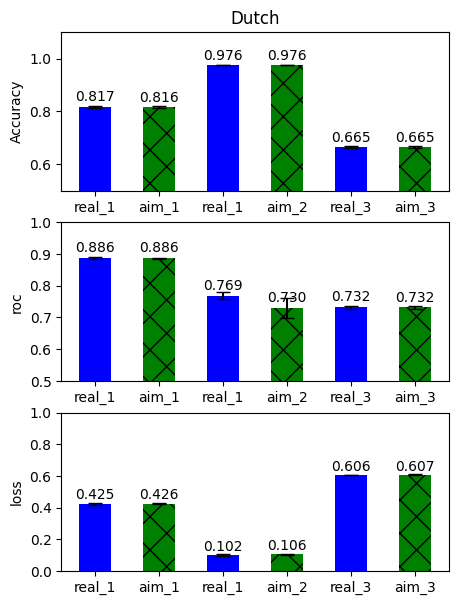}}
\label{fig:dutch_3_classifer}
\end{subfigure}
\end{minipage}
\hspace{5em}
\begin{minipage}{.22\textwidth}
\begin{subfigure}[Adult Dataset]
{\includegraphics[scale=0.33]{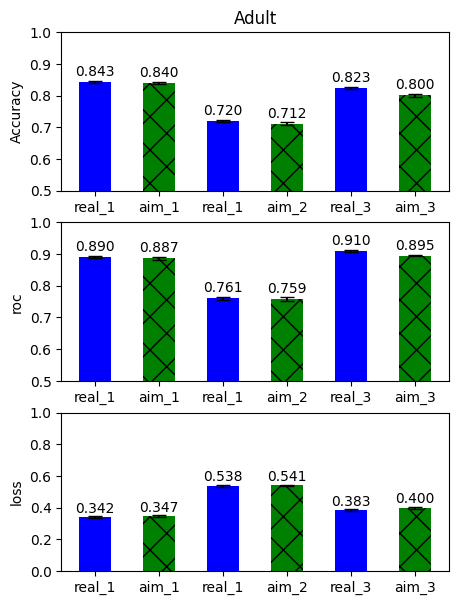}}
\label{fig:adult_3_classifer}
\end{subfigure}
\end{minipage}
\hspace{5em}
\begin{minipage}{.22\textwidth}
\begin{subfigure}[Law Dataset]
{\includegraphics[scale=0.33]{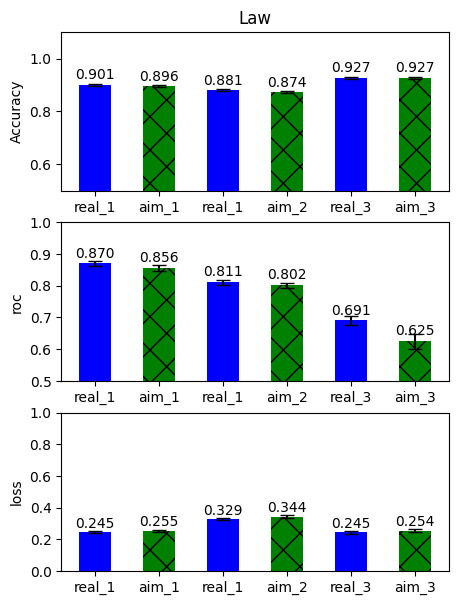}}
\label{fig:law_3_classifer}
\end{subfigure}
\end{minipage}
\caption{We train the three classifier models on each dataset and their synthetic data generated by AIM with privacy budget, epsilon=1. Dataset \{Adult, Churn, Law\}, three models are trained to classify three different target features: Dutch: \{'occupation', 'prev\_residence\_place', 'sex'\}, Adult: \{'income>50K', 'sex', 'relationship\},   Law: \{'pass\_bar', 'race', 'fulltime'\}. real\_1 and aim\_1 show results when classifying the first feature, and trained on real data, synthetic data from AIM, respectively; real\_2 and aim\_2 show results when classifying the 2nd feature, and trained on real data, synthetic data from AIM, respectively; real\_3 and aim\_3 show results when classifying the 3rd feature, and trained on real data, synthetic data from AIM, respectively.}
\label{fig: private-pgm vs aim}
\vskip -0.1in
\end{figure*}

\else{}
\fi

\end{document}